\documentclass{article}
\usepackage{amsmath,amsfonts,amssymb}
\usepackage{fullpage}
\usepackage{enumerate}
\usepackage[pdfusetitle]{hyperref}
\usepackage{mathtools}
\usepackage{natbib}
\usepackage{color}
\usepackage[linesnumbered,ruled,vlined]{algorithm2e}
\usepackage{graphicx}
\usepackage{authblk}
\usepackage{caption}
\usepackage{subfigure}
\usepackage{algorithmic}
\usepackage{wrapfig}
\usepackage{setspace}

\DeclarePairedDelimiter\paren{\lparen}{\rparen}

\DeclarePairedDelimiter\set{\{}{\}}

\definecolor{orange}{rgb}{1,0.5,0}

\def\bydef{:=}
\def\F{\mathcal{F}}
\def\G{\mathcal{G}}

\def\X{\mathcal{X}}

\def\Z{\mathcal{Z}}

\def\R{\mathcal{R}}
\def\G{\mathcal{G}}

\def\M{\mathcal{M}}

\def\P{\mathcal{P}}
\def\I{\mathcal{I}}
\def\J{\mathcal{J}}
\newcommand{\oo}[1]{\text{\sf 1}_{[#1]}}

\def\P{\mathbb{P}}

\newcommand{\e}[1]{\mathbb{E}\left[#1 \right]}
\newcommand{\ee}[2]{\mathbb{E}_{#1}\left[#2 \right]}

\newtheorem{remark}{Remark}

\newtheorem{theorem}{Theorem}
\newtheorem{lemma}{Lemma}

\newtheorem{corollary}{Corollary}

\newcommand{\halmos}{\rule[-0.4mm]{2.0mm}{3.2mm}}
\newenvironment{proof}{\par\noindent{\bf Proof\ }}{\hfill\halmos\\[2mm]}

\usepackage{epstopdf}
\title{AdaGAN: Boosting Generative Models}
\author[1]{Ilya Tolstikhin}
\author[2]{Sylvain Gelly}
\author[2]{Olivier Bousquet}
\author[1]{Carl-Johann Simon-Gabriel}
\author[1]{Bernhard Sch\"olkopf}
\affil[1]{Max Planck Institute for Intelligent Systems}
\affil[2]{Google Brain}
\date{}
\begin{document}
\maketitle
\begin{abstract}
Generative Adversarial Networks (GAN) \cite{goodfellow2014generative}
are an effective method for training generative models of complex data
such as natural images. However, they are notoriously
hard to train and can suffer from the problem of {\em missing modes}
where the model is not able to produce examples in certain regions of the
space.
We propose an iterative procedure, called {\em AdaGAN},
where at every step we add a new component
into a mixture model by running a GAN algorithm on a reweighted sample.
This is inspired by {\em boosting} algorithms,
where many potentially weak individual predictors are greedily aggregated to
form a strong composite predictor.
We prove that such an incremental
procedure leads to convergence to the
true distribution in a finite number of steps if each step is optimal,
and convergence at an exponential rate otherwise.
We also illustrate experimentally that this procedure addresses
the problem of missing modes.
\end{abstract}


\section{Introduction}
Imagine we have a large corpus, containing unlabeled pictures of animals, and our task is to build a generative probabilistic model of the data.
We run a recently proposed algorithm and end up with a model which produces impressive pictures of cats and dogs, but not a single giraffe.
A natural way to fix this would be to manually remove all cats and dogs from the training set and run the algorithm on the updated corpus.
The algorithm would then have no choice but to produce new animals and, by iterating this process until there's only giraffes left in the training set, we would arrive at a model generating giraffes (assuming sufficient sample size). 
At the end, we aggregate the models obtained by building a mixture model.
Unfortunately, the described meta-algorithm requires manual work for removing certain pictures from the \emph{unlabeled} training set at every iteration.

Let us turn this into an automatic approach, and rather than including or excluding a picture, put continuous weights on them. To this end, we train a binary classifier to separate ``true'' pictures of the original corpus from the set of ``synthetic'' pictures generated by the mixture of \emph{all the models} trained so far.
We would expect the classifier to make \emph{confident} predictions for the true pictures of animals missed by the model (giraffes), because there are no synthetic pictures nearby to be confused with them.
By a similar argument, the classifier should make less confident predictions for the true pictures containing animals already generated by one of the trained models (cats and dogs).
For each picture in the corpus, we can thus use the classifier's confidence to compute a weight which we use for that picture in the next iteration, to be performed on the re-weighted dataset.

The present work provides a principled way to perform this re-weighting, with theoretical guarantees showing that the resulting mixture 
 models indeed approach the true data distribution.\footnote{Note that the term ``mixture'' should not be interpreted to imply
that each component models only one mode: the models to be combined into a
mixture can themselves cover multiple modes.}

Before discussing how to build the mixture, let us consider the question of building a single
generative model.
A recent trend in modelling high dimensional data such as natural images is
to use neural networks~\cite{KW14, goodfellow2014generative}.
One popular approach are {\em Generative Adversarial Networks}~(GAN)~\cite{goodfellow2014generative}, where the generator is trained adversarially
against a classifier, which tries to differentiate the true from the generated data.
While the original GAN algorithm often produces realistically looking data, 
several issues were reported in the literature,
among which the \emph{missing modes problem}, where the generator converges to only one or a few modes of the data
distribution, thus not providing enough variability in the generated data.
This seems to match the situation described earlier, which is why we will most often illustrate our algorithm with a GAN as the underlying base generator.
We call it \emph{AdaGAN}, for Adaptive GAN, but
we could actually use any other generator: a Gaussian mixture model, a VAE \cite{KW14}, a WGAN \cite{wgan},  or even an unrolled \cite{MPPS2017}
or mode-regularized GAN~\cite{che2016mode},
which were both already specifically developed to tackle the missing mode problem.
Thus, we do not aim at improving the original GAN or any other generative algorithm.
We rather propose and analyse a meta-algorithm that can be used on top of any of them.
This meta-algorithm is similar in spirit to AdaBoost \cite{FS97} in the sense
that each iteration corresponds to learning a ``weak'' generative model (e.g., GAN) with respect to
a re-weighted data distribution. The weights change over time to focus on the ``hard''
examples, i.e.\:those that the mixture has not been able to properly generate
so far.

\subsection{Boosting via Additive Mixtures}
\label{sec:intro-our-approach}
Motivated by the problem of missing modes, in this work we propose to use multiple generative models
combined into a mixture. These generative models are trained iteratively
by adding, at each step, another model to the mixture that should hopefully
cover the areas of the space not covered by the previous mixture components.\footnote{Note that the term ``mixture'' should not be interpreted to imply
that each component models only one mode: the models to be combined into a
mixture can themselves cover multiple modes already.}
We show analytically that the optimal next mixture component can be obtained by
reweighting the true data, and thus propose to use the reweighted data distribution
as the target for the optimization of the next mixture components.
This leads us naturally to a \emph{meta-algorithm}, which is similar in spirit to AdaBoost in the sense
that each iteration corresponds to learning a ``weak'' generative model (e.g., GAN) with respect to
a reweighted data distribution. The latter adapts over time to focus on the ``hard''
examples, i.e.\:those that the mixture has not been able to properly generate
thus far.


Before diving into the technical details we provide an informal intuitive discussion of our new meta-algorithm, which we call \emph{AdaGAN} (a shorthand for Adaptive GAN, similar to AdaBoost).
The pseudocode is presented in Algorithm \ref{alg:AdaGAN}.

On the first step we run the GAN algorithm (or some other generative model) in the usual way and initialize our generative model with the resulting generator $G_1$.
On every $t$-th step we 
(a)~pick the mixture weight $\beta_t$ for the next component,
(b)~update weights $W_t$ of examples from the training set in such a way to bias the next component towards ``hard'' ones, not covered by the current mixture of generators $G_{t-1}$,
(c)~run the GAN algorithm, this time importance sampling mini-batches according to the updated weights~$W_t$, resulting in a new generator $G^c_t$, 
and finally (d) update our mixture of generators $G_t = (1-\beta_t) G_{t-1} + \beta_t G^c_t$ (notation expressing the mixture of $G_{t-1}$ and $G^c_t$ with probabilities $1-\beta_t$ and $\beta_t$).
This procedure outputs $T$ generator functions $G_1^c,\dots,G_T^c$ and $T$ corresponding non-negative weights $\alpha_1,\dots,\alpha_T$, which sum to one.
For sampling from the resulting model we first define a generator $G_i^c$, by sampling the index $i$ from a multinomial distribution with parameters $\alpha_1,\dots,\alpha_T$, 
and then we return $G^c_i(Z)$, where $Z\sim P_Z$ is a standard latent noise variable used in the GAN literature.

\begin{algorithm}\captionsetup{labelfont={sc,bf}, labelsep=newline}
\caption{AdaGAN, a meta-algorithm to construct a ``strong'' mixture of $T$ individual GANs, trained sequentially.
The mixture weight schedule ChooseMixtureWeight and the training set
reweighting schedule UpdateTrainingWeights should be provided by the user.
Section \ref{sec:adaGAN} gives a complete instance of this family.}
\label{alg:AdaGAN}
\begin{spacing}{1.4}
\renewcommand{\algorithmicrequire}{\textbf{Input:}}
\renewcommand{\algorithmicensure}{\textbf{Output:}}
  \begin{algorithmic}
  \REQUIRE{Training sample $S_N:=\{X_1,\dots,X_N\}$.}
  \ENSURE{Mixture generative model $G=G_T$.}
  \STATE{
  Train vanilla GAN:
  
  $W_1 = (1/N, \dots, 1/N)$  
  
  $G_1 = \mathrm{GAN}(S_N, W_t)$
  }
  \FOR{$t=2,\dots,T$}
  \STATE{
  \emph{\#Choose a mixture weight for the next component}

  $\beta_t = \mathrm{ChooseMixtureWeight}(t)$

  \emph{\#Update weights of training examples}

  $W_t = \mathrm{UpdateTrainingWeights}(G_{t-1}, S_N, \beta_t)$

  \emph{\#Train $t$-th ``weak'' component generator $G^c_t$}

  $G^c_t = \mathrm{GAN}(S_N, W_t)$

  \emph{\#Update the overall generative model}

  \emph{\#Notation below means forming a mixture of $G_{t-1}$ and $G^c_t$.}

  $G_t = (1-\beta_t) G_{t-1} + \beta_t G^c_t$
  }
  \ENDFOR
  \end{algorithmic}
\end{spacing}
\end{algorithm}

The effect of the described procedure is illustrated in a toy example in Figure~\ref{fig:missing_modes}.
On the~left images, the~red dots are the training (true data) points,
the~blue dots are points sampled from the model mixture of generators~$G_t$. 
The~background colour gives the density of the distribution corresponding to~$G_t$, non zero around the generated points, (almost) zero everywhere else.
On the right images, the color corresponds to the weights of training points, following the reweighting scheme proposed in this work.
The top row corresponds to the first iteration of AdaGAN, and the bottom row to the second iteration.
After the first iteration (the result of the vanilla GAN), we see that only the top left mode is covered,
while the three other modes are not covered at all. 
The~new weights (top right) show that the examples from covered mode
are aggressively downweighted. 
After the second iteration (bottom left), the combined generator can then generate two modes.

Although motivated by GANs, we cast our results in the general framework of
the minimization of an \hbox{$f$-divergence} (cf.\ \cite{nowozin2016f}) with respect to an additive mixture of
distributions. 
We also note that our approach may be combined with different ``weak'' generative models, including but not limited to GAN.

\begin{figure}[h]
    \centering
    \begin{minipage}[t]{0.45\textwidth}
    \includegraphics[width=\linewidth]{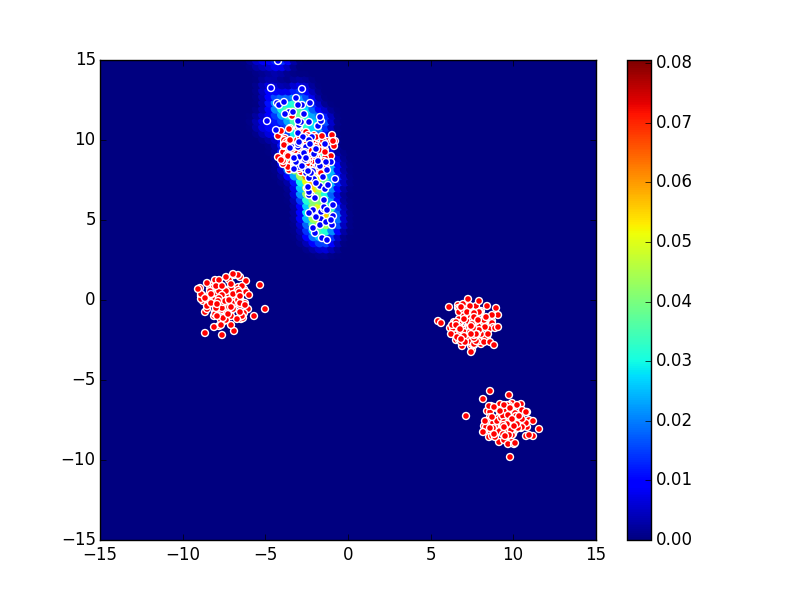}
    \end{minipage}
    \begin{minipage}[t]{0.45\textwidth}
    \includegraphics[width=\linewidth]{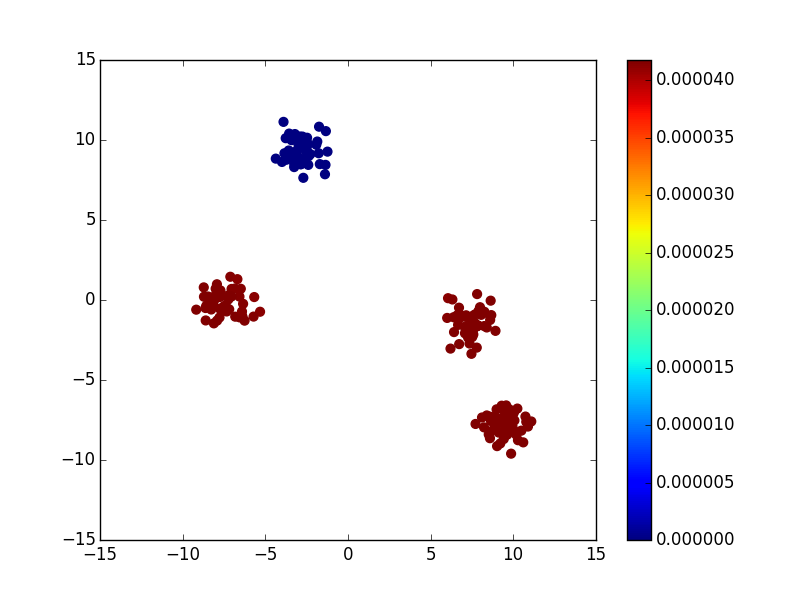}
    \end{minipage}
    
    \vspace{.1cm}
    \begin{minipage}[t]{0.45\textwidth}
    \includegraphics[width=\linewidth]{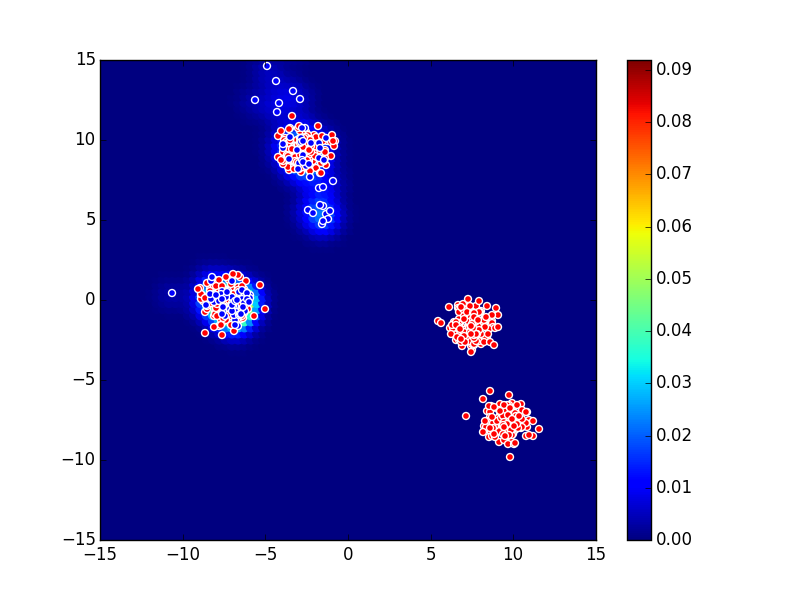}
    \end{minipage}
    \begin{minipage}[t]{0.45\textwidth}
    \includegraphics[width=\linewidth]{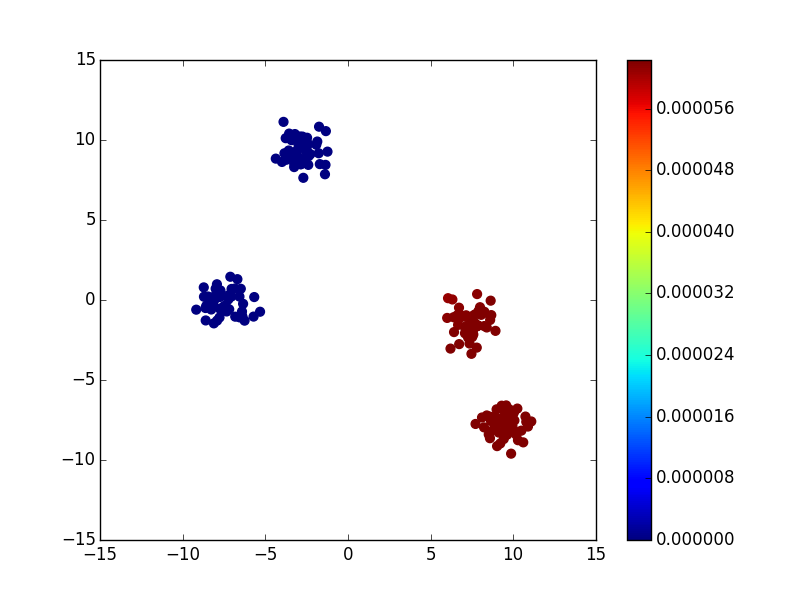}
    \end{minipage}
    \caption{A toy illustration of the missing mode problem and the effect of sample
reweighting, following the discussion in Section \ref{sec:intro-our-approach}. On the~left images, the~red dots are the training (true data) points,
the~blue dots are points sampled from the model mixture of generators~$G_t$. On the right images, the color corresponds to the weights of training points, following the reweighting scheme proposed in this work.
The top row corresponds to the first iteration of AdaGAN, and the bottom row to the second iteration.}
    \label{fig:missing_modes}
\end{figure}

\subsection{Related Work}
Several authors \cite{welling2002self, tu2007learning, grover2016boosted}
have proposed to use boosting techniques in the context of
density estimation by incrementally adding components in the log domain.
In particular, the work of Grover and Ermon \cite{grover2016boosted}, done in parallel to and independent of ours, is applying this idea to GANs.
A major downside of these approaches is that the resulting mixture is
a product of components and sampling from such a model is nontrivial
(at least when applied to GANs where the model density is not expressed
analytically) and
requires to use techniques such as Annealed Importance Sampling \cite{N01} for the
normalization.

Rosset and Segal \cite{rosset2002boosting} proposed to use an
additive mixture model in the case where the log likelihood can be computed.
They derived the update rule via computing the
steepest descent direction when adding a component with
infinitesimal weight. This leads to an update rule which is degenerate if the
generative model can produce arbitrarily concentrated distributions (indeed
the optimal component is just a Dirac distribution) which is thus not suitable
for the GAN setting.
Moreover, their results do not apply once the weight $\beta$ becomes non-infinitesimal.
In contrast, for any fixed weight of the new component our approach gives the
overall optimal update (rather than just the best
direction), and applies to any $f$-divergence.
Remarkably, in both theories, improvements of the mixture are guaranteed only if the new ``weak'' learner is still good enough (see Conditions \ref{eq:variational-surrogate}\&\ref{eq:gamma})

Similarly, Barron and Li \cite{barron1997mixture} studied the construction of mixtures minimizing the Kullback divergence and proposed a greedy procedure for doing so. They also proved that under certain conditions, finite mixtures can approximate arbitrary mixtures at a rate $1/k$ where $k$ is the number of components in the mixture when the weight of each newly added component is $1/k$. These results are specific to the Kullback divergence but are consistent with our more general results.

Wang et al.\:\cite{wang2016ensembles} propose an additive procedure similar to
ours but with a different reweighting scheme, which is not motivated by a theoretical analysis of optimality conditions.
On every new iteration the authors propose to run GAN on the top $k$ training examples with maximum value of the discriminator from the last iteration.
Empirical results of Section \ref{sec:experiments} show that this heuristic often fails to address the missing modes problem.

Finally, many papers investigate completely different
approaches for addressing the same issue by directly modifying the training objective of an individual GAN.
For instance, Che et al.\:\cite{che2016mode} add an autoencoding cost to the training objective of GAN,
while Metz et al.\:\cite{MPPS2017} allow the generator to ``look few steps ahead'' when making a gradient step.

\bigskip
The paper is organized as follows.
In Section \ref{sec:theory} we present our main theoretical results regarding optimization of mixture models under general $f$-divergences.
In particular we show that it is possible to build an optimal mixture in
an incremental fashion, where each additional component is obtained by
applying a GAN-style procedure with a reweighted distribution. 
In Section \ref{sec:optimal-convergence-analysis} we show that if the GAN optimization at each step is perfect, the process converges to the true data distribution at exponential rate (or even in a \emph{finite number of steps}, for which we provide a necessary and sufficient condition).
Then we show in Section \ref{sec:weak-to-strong} that imperfect GAN solutions still lead to the exponential rate of convergence under certain ``weak learnability'' conditions.
These results naturally lead us to a new boosting-style iterative procedure for constructing generative models, which is combined with GAN in Section~\ref{sec:adaGAN}, resulting in a new algorithm called \emph{AdaGAN}.
Finally, we report initial empirical results in Section \ref{sec:experiments}, where we compare AdaGAN with several benchmarks, including original GAN, uniform mixture of multiple independently trained GANs, and iterative procedure of Wang et al.\:\cite{wang2016ensembles}.

\section{Minimizing $f$-divergence with  Additive Mixtures}
\label{sec:theory}
In this section we derive a general result on the minimization of
$f$-divergences over mixture models.

\subsection{Preliminaries and notations}
In this work we will write $P_{d}$ and $P_{model}$ to denote a
real data distribution and our approximate model distribution, respectively,
both defined over the data space $\X$.

\paragraph{Generative Density Estimation}
In the generative approach to density estimation, instead of building a
probabilistic model of the data directly, one builds a function $G:\Z\to\X$
that transforms a fixed probability distribution $P_Z$
(often called the {\em noise} distribution)
over a latent space $\Z$ into a distribution over $\X$.
Hence $P_{model}$ is the pushforward of $P_Z$, i.e.\:$P_{model}(A)=P_Z(G^{-1}(A))$.
Because of this definition, it is generally impossible to compute the density
$dP_{model}(x)$, hence it is not possible to compute the log-likelihood of the
training data under the model.
However, if $P_Z$ is a distribution from which one can sample, it is easy to
also sample from $P_{model}$ (simply sampling from $P_Z$ and applying $G$ to
each example gives a sample from $P_{model}$).

So the problem of generative density estimation becomes a problem of finding
a function $G$ such that $P_{model}$ looks like $P_d$ in the sense that samples
from $P_{model}$ and from $P_d$ look similar.
Another way to state this problem is to say that we are given a measure of
similarity between distributions $D(P_{model} \| P_d)$ which can be estimated
from samples of those distributions, and thus approximately minimized over a
class $\G$ of functions.

\paragraph{$f$-Divergences}
In order to measure the agreement between the model distribution and the true
distribution of the data we will use an $f$-divergence defined in the following
way:
\begin{equation}
\label{eq:f-divergence}
D_f(Q \| P)\bydef \int f\left(\frac{dQ}{dP}(x)\right)dP(x)
\end{equation}
for any pair of distributions $P,Q$ with densities $dP$, $dQ$ with respect to
some dominating reference measure~$\mu$.
In this work we assume that the function $f$ is convex, defined on $(0,\infty)$, and satisfies
$f(1)=0$.
The definition of $D_f$ holds for both continuous and discrete probability measures and does not depend on specific choice of $\mu$.\footnote{
The integral in \eqref{eq:f-divergence} is well defined (but may take infinite values) even if $P(dQ = 0) > 0$ or $Q(dP = 0) > 0$. In this case the integral is understood as $D_f(Q \| P) = \int f(dQ/dP)\oo{dP(x) > 0, dQ(x) > 0}dP(x) + f(0) P(dQ = 0) + f^{\circ}(0) Q(dP = 0)$, where both $f(0)$ and $f^{\circ}(0)$ may take value $\infty$ \citep{LM08}.
This is especially important in case of GAN, where it is impossible to constrain $P_{model}$ to be absolutely continuous with respect to $P_d$ or vice versa.
}
It is easy to verify that $D_f\ge 0$ and it is equal to $0$ when $P=Q$.
Note that $D_f$ is not symmetric, but $D_f(P\|Q) = D_{f^{\circ}}(Q\|P)$ for $f^{\circ}(x):=xf(1/x)$ and any $P$ and $Q$.
The $f$-divergence is {\em symmetric} when $f(x)=f^{\circ}(x)$ for all $x\in(0,\infty)$, as in this case $D_f(P,Q)=D_f(Q,P)$.

We also note that the divergences corresponding to $f(x)$ and $f(x)+ C\cdot (x-1)$ are identical for any constant~$C$.
In some cases, it is thus convenient to work with $f_0(x) := f(x) - (x-1)f'(1)$, (where $f'(1)$ is any subderivative
of $f$ at $1$) as $D_f(Q\|P) = D_{f_0}(Q\|P)$ for all $Q$ and $P$, while $f_0$ is nonnegative, nonincreasing on $(0,1]$, and nondecreasing on $(1,\infty)$.
In the remainder, we will denote by $\F$ the set of functions that are suitable
for $f$-divergences, i.e. the set of functions of the form $f_0$ for any convex
$f$ with $f(1)=0$.

Classical examples of $f$-divergences include the Kullback-Leibler divergence
(obtained for $f(x)=-\log x$, $f_0(x)=-\log x+x-1$),
the reverse Kullback-Leibler divergence (obtained
for $f(x)=x\log x$, $f_0(x)=x\log x-x+1$), the Total Variation distance
($f(x)=f_0(x)=|x-1|$), and the
Jensen-Shannon divergence ($f(x)=f_0(x)=-(x+1)\log \frac{x+1}{2}+x\log x$).
More details can be found in Appendix \ref{appendix:f-div}.
Other examples can be found in
\cite{nowozin2016f}.
For further details on $f$-divergences we refer to Section 1.3 of~\citep{LM08} and~\cite{RW11}.

\paragraph{GAN and $f$-divergences}
We now explain the connection between the GAN algorithm and $f$-divergences.
The original GAN algorithm \cite{goodfellow2014generative} consists in optimizing the following criterion:
\begin{equation}\label{eq:gan}
 \min_{G}\max_{D} \ee{P_d}{\log D(X)} + \ee{P_Z}{\log \paren*{1-D(G(Z))}}\,,
\end{equation}
where $D$ and $G$ are two functions represented by neural networks, and this optimization is actually performed on a pair of samples
(one being the training sample, the other one being created from the chosen distribution~$P_Z$), which corresponds to approximating
the above criterion by using the empirical distributions.
For a fixed $G$, it has been shown in \cite{goodfellow2014generative} that the optimal $D$ for \eqref{eq:gan} is given by
$D^*(x)=\frac{dP_d(x)}{dP_d(x)+dP_g(x)}$ and plugging this optimal value into \eqref{eq:gan} gives the following:
\begin{equation}\label{eq:gan2}
 \min_{G} -\log(4) + 2JS(P_d\,\|\, P_g)\,,
\end{equation}
where $JS$ is the Jensen-Shannon divergence.
Of course, the actual GAN algorithm uses an approximation to $D^*$ which is computed by training a neural network on a sample,
which means that the GAN algorithm can be considered to minimize an approximation of \eqref{eq:gan2}\footnote{Actually the criterion that is
minimized is an empirical version of a lower bound of the Jensen-Shannon divergence.}.
This point of view can be generalized by plugging another $f$-divergence into \eqref{eq:gan2}, and it turns out that
other $f$-divergences can be written as the solution to a maximization of a criterion similar to \eqref{eq:gan}.
Indeed, as demonstrated in \cite{nowozin2016f}, any $f$-divergence between $P_d$ and $P_g$ 
can be seen as the optimal value of a quantity of the form $\ee{P_d}{f_1(D(X))}+\ee{P_g}{f_2(D(G(Z)))}$ for appropriate $f_1$ and $f_2$,
and thus can be optimized by the same adversarial training technique.

There is thus a strong connection between adversarial training of generative models and minimization of $f$-divergences, and this is why
we cast the results of this section in the context of general $f$-divergences.

\paragraph{Hilbertian Metrics}
As demonstrated in \cite{fuglede2004jensen, hein2005hilbertian}, several commonly used
symmetric $f$-divergences
are Hilbertian metrics, which in particular means that their square root
satisfies the triangle inequality.
This is true for the Jensen-Shannon divergence\footnote{which means such a property
can be used in the context of the original GAN algorithm.}
as well as for the Hellinger
distance and the Total Variation among others.
We will denote by $\F_H$ the set of $f$ functions such that $D_f$ is a Hilbertian
metric. For those divergences, we have $D_f(P\|Q)\le (\sqrt{D_f(P\|R)}+\sqrt{D_f(R\|Q)})^2$.

\paragraph{Generative Mixture Models}
In order to model complex data distributions, it can be convenient to use
a mixture model of the following form:
\begin{equation}\label{eq:mixture}
P^T_{model} := \sum_{i=1}^T \alpha_i P_{i},
\end{equation}
where $\alpha_i \geq 0$, $\sum_i \alpha_i = 1$, and each of the $T$ components
is a generative density model.
This is very natural in the generative context, since sampling
from a mixture corresponds to a two-step sampling, where one first picks the
mixture component (according to the multinomial distribution whose parameters
are the $\alpha_i$) and then samples from it.
Also, this allows to construct complex models from simpler ones.

\subsection{Incremental Mixture Building}
As discussed earlier, in the context of generative modeling, we are given a
measure of similarity between distributions. We will restrict ourselves to the
case of $f$-divergences. Indeed, for any $f$-divergence, it is possible
(as explained for example in \cite{nowozin2016f}) to
estimate $D_f(Q\,\|\,P)$ from two samples (one from $Q$, one from~$P$) by
training a ``discriminator'' function, i.e. by solving an optimization problem
(which is a binary classification problem in the case where the divergence
is symmetric\footnote{
One example of such a setting is running GANs, which are known to approximately
minimize the Jensen-Shannon divergence.}).
It turns out that the empirical estimate $\hat{D}$ of $D_f(Q\,\|\,P)$
thus obtained provides a criterion for optimizing $Q$ itself.
Indeed, $\hat{D}$ is a function of $Y_1,\ldots,Y_n\sim Q$ and
$X_1,\ldots,X_n\sim P$, where $Y_i=G(Z_i)$ for some mapping function $G$.
Hence it is possible to optimize $\hat{D}$ with respect to $G$ (and in particular
compute gradients with respect to the parameters of $G$ if $G$ comes from a
smoothly parametrized model such as a neural network).

In this work we thus assume that, given an i.i.d.\:sample from any unknown
distribution $P$ we can construct a simple model $Q\in\G$ which approximately minimizes
\begin{equation}
\label{eq:what-we-can-do}
	\min_{Q \in \G} D_f(Q\, \| \, P).
\end{equation}

Instead of just modelling the data with a single distribution, we now want to
model it with a mixture of the form \eqref{eq:mixture} where each $P_i$
is obtained by a training procedure of the form \eqref{eq:what-we-can-do}
with (possibly) different target distributions~$P$ for each $i$.

A natural way to build a mixture is to do it incrementally:
we train the first model $P_1$ to minimize $D_f(P_1 \, \| \, P_d)$
and set the corresponding weight to $\alpha_1 = 1$, leading to $P^1_{model} = P_1$.
Then after having trained $t$ components $P_1,\dots, P_t \in \G$ we can form
the $(t+1)$-st mixture model by adding a new component $Q$ with weight $\beta$ as
follows:
\begin{equation}
\label{eq:mixture-model}
P^{t+1}_{model} := \sum_{i=1}^t (1-\beta)\alpha_i P_{i} + \beta Q.
\end{equation}
We are going to choose $\beta \in [0,1]$ and $Q\in\G$ greedily, while keeping all the
other parameters of the generative model fixed, so as to minimize
\begin{equation}\label{eq:idealObjective}
D_f((1-\beta)P_g + \beta Q \, \| \, P_d),
\end{equation}
where we denoted $P_g:= P^t_{model}$ the current generative mixture model before adding the new component.

We do not necessarily need to find the optimal $Q$ that minimizes \eqref{eq:idealObjective}
at each step. Indeed, it would be sufficient to find some $Q$ which allows to
build a slightly better approximation of $P_d$.
This means that a more modest goal could be to find $Q$ such that, for some
$c<1$,
\begin{equation}\label{eq:modestObjective}
D_f((1-\beta)P_g + \beta Q \, \| \, P_d) \le c\cdot D_f(P_g \, \| \, P_d)\,.
\end{equation}

However, we observe that this greedy approach has a significant drawback in
practice. Indeed, as we build up the mixture, we need to make $\beta$ decrease
(as $P_{model}^t$ approximates $P_d$ better and better, one should
make the correction at each step smaller and smaller). Since we are approximating
\eqref{eq:idealObjective} using samples from both distributions, this means that
the sample from the mixture will only contain a fraction $\beta$ of examples
from $Q$. So, as $t$ increases, getting meaningful information from a sample
so as to tune $Q$ becomes harder and harder (the information is ``diluted'').


To address this issue, we propose to optimize an upper bound on \eqref{eq:idealObjective}
which involves a term of the form $D_f(Q\,\|\, Q_0)$ for some distribution $Q_0$,
which can be computed as a reweighting of the original data distribution~$P_d$.

In the following sections we will analyze the properties of \eqref{eq:idealObjective}
(Section \ref{sect:direct}) and derive upper bounds that provide practical
optimization criteria for building the mixture (Section \ref{sect:upperbound}).
We will also show that under certain assumptions, the minimization of the
upper bound will lead to the optimum of the original criterion.

This procedure is reminiscent of the AdaBoost algorithm \citep{FS97}, which combines multiple \emph{weak} predictors into one very accurate \emph{strong} composition.
On each step AdaBoost adds one new predictor to the current composition, which is trained to minimize the binary loss on the reweighted training set.
The weights are constantly updated in order to bias the next weak learner towards ``hard'' examples, which were incorrectly classified during previous stages.

\subsection{Upper Bounds}
\label{sect:upperbound}

Next lemma provides two upper bounds on the divergence of the mixture
in terms of the divergence of the additive component $Q$ with respect to
some reference distribution $R$.

\begin{lemma}\label{le:aux-joint-convexity}
Let $f\in\F$.
Given two distributions $P_d,P_g$ and some $\beta\in[0,1]$,
for any distribution $Q$ and any distribution $R$ such that $\beta dR\le dP_d$, we have
\begin{equation}\label{up1}
D_f\paren*{(1-\beta)P_g +\beta Q\,\|\,P_d} \le \beta D(Q\,\|\,R) +
(1-\beta) D_f\paren*{P_g\,\|\, \frac{P_d-\beta R}{1-\beta}}\,.
\end{equation}
If furthermore $f\in\F_H$, then, for any $R$, we have
\begin{equation}\label{up2}
D_f\paren*{(1-\beta)P_g +\beta Q\,\|\,P_d} \le \paren*{
\sqrt{\beta D_f(Q\,\|\,R)} +
\sqrt{D_f\paren*{(1-\beta)P_g +\beta R\,\|\,P_d}} }^2\,.
\end{equation}
\end{lemma}
\begin{proof}
For the first inequality, we use the fact that $D_f$ is jointly convex.
We write
$P_d = (1-\beta)\frac{P_d-\beta R}{1-\beta} + \beta R$ which is a convex
combination of two distributions when the assumptions are satisfied.

The second inequality follows from using the triangle inequality for the square
root of the Hilbertian metric $D_f$ and using convexity of $D_f$ in its first
argument.
\end{proof}

We can exploit the upper bounds of Lemma \ref{le:aux-joint-convexity}
by introducing some well-chosen distribution $R$ and minimizing with respect to
$Q$. A natural choice for $R$ is a distribution that minimizes the last term
of the upper bound (which does not depend on $Q$).

\subsection{Optimal Upper Bounds}
\label{sect:direct}

In this section we provide general theorems about the optimization of the
right-most terms in the upper bounds of Lemma \ref{le:aux-joint-convexity}.

For the upper bound \eqref{up2}, this means we need to find $R$ minimizing
$D_f\paren*{(1-\beta)P_g +\beta R\,\|\,P_d}$.
The solution for this problem is given in the following
theorem.
\begin{theorem}\label{main}
For any $f$-divergence $D_f$, with $f\in\F$ and $f$ differentiable,
any fixed distributions $P_d,P_g$, and any
$\beta\in(0,1]$, the solution to the following minimization problem:
\[
\min_{Q\in \P} D_f((1-\beta)P_g + \beta Q \, \| \, P_d),
\]
where $\P$ is a class of all probability distributions,
has the density
\[
dQ^*_{\beta}(x) = \frac{1}{\beta}\left(\lambda^* dP_d(x) - (1-\beta) dP_g(x) \right)_+
\]
for some unique $\lambda^*$ satisfying
$
\int dQ^*_{\beta} = 1.
$
Furthermore, $\beta\leq\lambda^*\leq \min(1, \beta/\delta)$, where $\delta := P_d(dP_g = 0)$.
Also, $\lambda^*=1$ if and only if $P_d((1-\beta)dP_g > dP_d)=0$, which is
equivalent to $\beta dQ_\beta^*= dP_d-(1-\beta)dP_g$.
\end{theorem}
\begin{proof}
See Appendix \ref{sec:proof-main}.
\end{proof}
\begin{remark}
The form of $Q^*_{\beta}$ may look unexpected at first glance: why not setting $dQ:=\bigl( dP_d - (1-\beta) dP_g\bigr)/\beta$, which would make arguments of the $f$-divergence identical?
Unfortunately, it may be the case that $dP_d(X) < (1-\beta) dP_g(X)$ for some of $X\in\X$, leading to the negative values of $dQ$.
\end{remark}

For the upper bound \eqref{up1}, we need to minimize
$D_f\paren*{P_g\,\|\, \frac{P_d-\beta R}{1-\beta}}$.
The solution is given in the next theorem.
\begin{theorem}\label{th:opt2}
Given two distributions $P_d,P_g$ and some $\beta\in(0,1]$, assume
\[
P_d\left(dP_g= 0\right) < \beta.
\]
Let $f\in\F$. The solution to the minimization problem
\[
\min_{Q:\beta dQ\le dP_d}
D_f\paren*{P_g\,\|\, \frac{P_d-\beta Q}{1-\beta}}
\]
is given by the distribution
\[
dQ^{\dagger}_{\beta}(x) = \frac{1}{\beta}\paren*{dP_d(x)-\lambda^\dagger (1-\beta)dP_g(x)}_+
\]
for a unique $\lambda^\dagger\geq1$ satisfying $\int dQ^{\dagger}_{\beta} = 1$.
\end{theorem}
\begin{proof}
See Appendix \ref{sec:proof-opt2}.
\end{proof}

\begin{remark}
Notice that the term that we optimized in upper bound \eqref{up2} is exactly
the initial objective \eqref{eq:idealObjective}. So that Theorem \ref{main}
also tells us what the form of the optimal distribution is for the initial
objective.
\end{remark}
\begin{remark}
Surprisingly, in both Theorem \ref{main} and \ref{th:opt2},
the solution does not depend on the choice of the function~$f$, which means
that the solution is the same for {\em any} $f$-divergence.
This also means that by replacing $f$ by $f^\circ$, we get similar results for
the criterion written in the other direction, with again the same solution.
Hence the order in which we write the divergence does not matter and the
optimal solution is optimal for both orders.
\end{remark}
\begin{remark}
Note that $\lambda^*$ is implicitly defined by a fixed-point equation.
In Section \ref{sec:lambda-compute} we will show how it can be computed
efficiently in the case of empirical distributions.
\end{remark}
\begin{remark}
Obviously, $\lambda^\dagger \geq \lambda^*$, where $\lambda^*$ was defined in Theorem \ref{main}.
Moreover, we have $\lambda^* \leq 1/\lambda^\dagger$.
Indeed, it is enough to insert $\lambda^\dagger = 1/\lambda^*$ into definition of $Q^\dagger_{\beta}$ and check that in this case $Q^\dagger_{\beta} \geq 1$.
\end{remark}

\subsection{Convergence Analysis for Optimal Updates}
In previous section we derived analytical expressions for the distributions $R$ minimizing last terms in upper bounds \eqref{up1} and \eqref{up2}.
Assuming $Q$ can perfectly match $R$, i.e.\:$D_f(Q\,\|\,R) = 0$, we are now interested in the convergence of the mixture \eqref{eq:mixture-model} to the true data distribution $P_d$ for $Q=Q^*_{\beta}$ or $Q=Q^\dagger_{\beta}$.

We start with simple results showing that
adding $Q_\beta^*$ or $Q_\beta^\dagger$ to the current mixture would yield
a strict improvement of the divergence.
\begin{lemma}\label{optimal-imp} 
Under the conditions of Theorem \ref{main}, we have
\begin{equation}
\label{eq:optimal-improvement-1}
D_f\bigl((1-\beta)P_g + \beta Q^*_{\beta} \,\big\|\, P_d \bigr) \leq
D_f\bigl((1-\beta)P_g + \beta P_d \,\big\|\, P_d \bigr) \leq
(1-\beta) D_f(P_g \,\|\, P_d).
\end{equation}
Under the conditions of Theorem \ref{th:opt2}, we have
\[
D_f\paren*{P_g \,\big\|\, \frac{P_d - \beta Q_\beta^\dagger}{1-\beta}} \leq
D_f(P_g \,\|\, P_d)\,,
\]
and
\[
D_f\bigl((1-\beta)P_g + \beta Q^\dagger_{\beta} \,\big\|\, P_d \bigr) \leq
(1-\beta) D_f(P_g \,\|\, P_d).
\]
\end{lemma}
\begin{proof}
The first inequality follows immediately from the optimality of $Q_\beta^*$
(hence the value of the objective at $Q_\beta^*$ is smaller than at $P_d$),
and the fact that $D_f$ is convex in its first argument and $D_f(P_d\|P_d)=0$.
The~second inequality follows from the optimality of $Q_\beta^\dagger$ (hence
the value of the objective at $Q_\beta^\dagger$ is smaller than its value
at $P_d$ which itself satisfies the condition $\beta dP_d \le dP_d$).
For the third inequality, we combine the second inequality with the first
inequality of Lemma \ref{le:aux-joint-convexity} (with $Q=R=Q_\beta^\dagger$).
\end{proof}

The upper bound \eqref{eq:optimal-improvement-1} of Lemma \ref{optimal-imp} can be refined if the ratio $dP_g/dP_d$ is almost surely bounded:
\begin{lemma}
Under the conditions of Theorem \ref{main}, if there exists $M>1$
such that
\[
P_d((1-\beta)dP_g > MdP_d) = 0
\]
then
\[
D_f\bigl((1-\beta)P_g + \beta Q^*_{\beta} \,\big\|\, P_d \bigr) \leq
f(\lambda^*) + \frac{f(M)(1-\lambda^*)}{M-1}.
\]
\end{lemma}
\begin{proof}
We use Inequality \eqref{eq:max2} of Lemma \ref{le:imp} with $X=\beta$,
$Y=(1-\beta)dP_g/dP_d$, and
$c=\lambda^*$.
We easily verify that $X+Y=((1-\beta)dP_g+\beta dP_d)/dP_d$ and
$\max(c,Y)=((1-\beta)dP_g+\beta dQ_\beta^*)/dP_d$ and both have expectation $1$ with
respect to $P_d$.
We thus obtain:
\[
D_f((1-\beta)P_g + \beta Q_\beta^* \, \| \, P_d) \le f(\lambda^*)+
\frac{f(M)-f(\lambda^*)}{M-\lambda^*}\paren*{1-\lambda^*}\,.
\]
Since $\lambda^*\le 1$ and $f$ is non-increasing on $(0,1)$ we get
\[
D_f((1-\beta)P_g + \beta Q_\beta^* \, \| \, P_d) \le f(\lambda^*) + \frac{f(M)(1-\lambda^*)}{M-1}.
\]
\end{proof}
\begin{remark}

This upper bound can be tighter than that of Lemma \ref{optimal-imp} when
$\lambda^*$ gets close to $1$. Indeed, for $\lambda^*=1$ the upper bound is
exactly $0$ and is thus tight,  while the upper bound of Lemma~\ref{optimal-imp} will not be zero in this case.
\end{remark}

\label{sec:optimal-convergence-analysis}

Imagine repeatedly adding $T$ new components to the current mixture $P_g$, where on every step we use the same weight $\beta$ and choose the components described in Theorem \ref{main}.
In this case Lemma \ref{optimal-imp} guarantees that the original objective value $D_f(P_g \,\|\, P_d)$ would be reduced at least to $(1-\beta)^TD_f(P_g \,\|\, P_d)$.
This exponential rate of convergence, which at first may look surprisingly good, is simply explained by the fact that $Q^*_{\beta}$ depends on the true distribution $P_d$, which is of course unknown.

Lemma \ref{optimal-imp} also suggests setting $\beta$ as large as possible.
This is intuitively clear: the smaller the $\beta$, the less we alter our current model $P_g$. As a consequence, choosing small $\beta$ when $P_g$ is far away from $P_d$ would lead to only minor improvements in objective \eqref{eq:idealObjective}.
In fact, the global minimum of \eqref{eq:idealObjective} can be reached by setting $\beta = 1$ and $Q = P_d$.
Nevertheless, in practice we may prefer to keep $\beta$ relatively small, preserving what we learned so far through $P_g$: for instance, when $P_g$ already covered part of the modes of $P_d$ and we want $Q$ to cover the remaining ones. 
We provide further discussions on choosing $\beta$ in Section \ref{sec:choosing-beta}.

In the reminder of this section we study the convergence of \eqref{eq:idealObjective} to $0$ in the case where we use the upper bound \eqref{up2} and the weight $\beta$ is fixed (i.e.\:the same value at each iteration). This analysis can easily be extended to a variable $\beta$.

\begin{lemma}\label{le:conv}
For any $f\in\F$ such that $f(x)\ne 0$ for $x\ne 1$, the following conditions
are equivalent:
\begin{enumerate}[(i)]
\item $P_d((1-\beta)dP_g > dP_d)=0$;
\item $D_f((1-\beta)P_g + \beta Q_\beta^* \,\|\, P_d) = 0$.
\end{enumerate}
\end{lemma}
\begin{proof}
The first condition is equivalent to $\lambda^*=1$ according to Theorem \ref{main}.
In this case, $(1-\beta)P_g+\beta Q_\beta^*=P_d$, hence the divergence is $0$.
In the other direction, when the divergence is $0$, since $f$ is strictly positive
for $x\ne 1$ (keep in mind that we can always replace $f$ by $f_0$ to get a
non-negative function which will be strictly positive if $f(x)\ne 0$ for $x\ne 1$),
this means that with $P_d$ probability $1$ we have the equality
$dP_d = (1-\beta)dP_g + \beta dQ_\beta^*$, which implies that
$(1-\beta)dP_g > dP_d$ with $P_d$ probability $1$ and also $\lambda^*=1$.
\end{proof}
%

This result tells that we can not perfectly match $P_d$ by adding a new mixture component to $P_g$ as long as there are points in the space where our current model $P_g$ severely over-samples.
As an example, consider an extreme case where $P_g$ puts a positive mass in a region outside of the support of $P_d$.
Clearly, unless $\beta = 1$, we will not be able to match $P_d$.

Finally, we provide a necessary and sufficient condition for the iterative process to converge to the data distribution $P_d$ in finite number of steps.
The criterion is based on the ratio $dP_1/dP_d$, where $P_1$ is the first component of our mixture model.

\begin{corollary}
Take \,any $f\in\F$ such that $f(x)\ne 0$ for $x\ne 1$.
Starting from $P^1_{model}=P_1$, update the model iteratively according to $P^{t+1}_{model}= (1-\beta)P^{t}_{model} + \beta Q^*_{\beta}$, where on every step $Q^*_{\beta}$ is as defined in Theorem \ref{main} with $P_g := P^{t}_{model}$.
In this case $D_f(P_{model}^t\,\|\, P_d)$
will reach $0$ in a finite number of steps if and only if there exists $M>0$
such that
\begin{equation}
\label{eq:convergence-criterion}
P_d((1-\beta)dP_1 > MdP_d) = 0\,.
\end{equation}
When the finite convergence happens, it takes at most $-\ln \max(M,1) /\ln (1-\beta)$
steps.
\end{corollary}
\begin{proof}
From Lemma \ref{le:conv}, it is clear that if $M\leq1$ the convergence happens after the first update. So let us assume $M>1$.
Notice that $dP_{model}^{t+1}=(1-\beta)dP_{model}^t+\beta dQ_\beta^* = \max(\lambda^* dP_d,
(1-\beta)dP_{model}^t)$
so that if $P_d((1-\beta)dP_{model}^t > MdP_d) = 0$, then
$P_d((1-\beta)dP_{model}^{t+1} > M(1-\beta)dP_d) = 0$.
This proves that \eqref{eq:convergence-criterion} is a sufficient condition.

Now assume the process converged in a finite number of steps. 
Let $P^t_{model}$ be a mixture right before the final step.
Note that $P^t_{model}$ is represented by $(1-\beta)^{t-1} P_1 + (1-(1-\beta)^{t-1}) P$ for certain probability distribution $P$.
According to Lemma \ref{le:conv} we have 
$P_d((1-\beta)dP^t_{model} > dP_d)=0$.
Together these two facts immediately imply \eqref{eq:convergence-criterion}.
\end{proof}
It is also important to keep in mind that even if \eqref{eq:convergence-criterion} is not satisfied the process still converges to the true distribution at exponential rate (see Lemma~\ref{optimal-imp} as well as Corollaries~\ref{th:main2} and \ref{thm:hilbertian} below)

\subsection{Weak to Strong Learnability}
\label{sec:weak-to-strong}
In practice the component $Q$ that we add to the mixture is not exactly $Q_\beta^*$
or $Q_\beta^\dagger$, but rather an approximation to them.
We need to show that if this approximation is good enough, then we retain the property that
\eqref{eq:modestObjective} is reached.
In this section we will show that this is indeed the case.

Looking again at Lemma \ref{le:aux-joint-convexity} we notice that the first
upper bound is less tight than the second one.
Indeed, take the optimal distributions provided by Theorems \ref{main} and \ref{th:opt2}
and plug them back as $R$ into the upper bounds of Lemma \ref{le:aux-joint-convexity}.
Also assume that $Q$ can match $R$ exactly, i.e.\:we can achieve $D_f(Q\,\|\,R)=0$.
In this case both sides of \eqref{up2} are equal to
$D_f((1-\beta)P_g+\beta Q_\beta^*\,\|\,P_d)$, which is the optimal value for the original objective \eqref{eq:idealObjective}.
On the other hand, \eqref{up1} does not become an equality and the r.h.s.\:is not the optimal one for \eqref{eq:idealObjective}.

This means that using \eqref{up2} allows to reach the optimal value of the original objective
\eqref{eq:idealObjective}, whereas using~\eqref{up1} does not.
However, this is not such a big issue since, as we mentioned earlier,
we only need to improve the mixture by adding the next component (we do not need
to add the optimal next component). So despite the solution of \eqref{eq:idealObjective} not
being reachable with the first upper bound, we will still show that
\eqref{eq:modestObjective} can be reached.

The first result provides sufficient conditions for strict improvements when we use the upper bound \eqref{up1}.
\begin{corollary}\label{th:main2}
Given two distributions $P_d,P_g$, and some $\beta\in(0,1]$, assume
\begin{equation}
\label{eq:condition-missing-modes}
P_d\left( \frac{dP_g}{dP_d}= 0\right) < \beta.
\end{equation}
Let $Q^\dagger_{\beta}$ be as defined in Theorem \ref{th:opt2}.
If $Q$ is a distribution satisfying
\begin{equation}
\label{eq:variational-surrogate}
D_f(Q \,\|\,Q^{\dagger}_{\beta}) \le \gamma D_f(P_g\,\|\,P_d)
\end{equation}
for $\gamma \in [0,1]$ then
\[
D_f\paren*{(1-\beta)P_g +\beta Q\,\|\,P_d} \le
(1 - \beta(1-\gamma)) D_f(P_g\,\|\,P_d).
\]
\end{corollary}
\begin{proof}
Immediately follows from combining Lemma \ref{le:aux-joint-convexity}, Theorem \ref{main}, and Lemma \ref{optimal-imp}.
\end{proof}
Next one holds for Hilbertian metrics and corresponds to the upper bound \eqref{up2}.
\begin{corollary}\label{thm:hilbertian}
Assume $f\in \mathcal{F}_H$, i.e.\: $D_f$ is a Hilbertian metric.
Take any $\beta \in (0,1]$, $P_d$, $P_g$, and let $Q^*_{\beta}$ be as defined in Theorem~\ref{main}.
If $Q$ is a distribution satisfying
\begin{equation}\label{eq:gamma}
D_f(Q \,\|\,Q^*_{\beta}) \le \gamma D_f(P_g\,\|\,P_d)
\end{equation}
for some $\gamma\in[0,1]$, then
\[
D_f\paren*{(1-\beta)P_g +\beta Q\,\|\,P_d} \le \paren*{\sqrt{\gamma\beta}+\sqrt{1-\beta}}^2 D_f(P_g\,\|\,P_d)\,.
\]
In particular, the right-hand side is strictly smaller than $D_f(P_g\,\|\, P_d)$
as soon as $\gamma <\beta / 4$ (and $\beta>0$).
\end{corollary}
\begin{proof}
Immediately follows from combining Lemma \ref{le:aux-joint-convexity}, Theorem \ref{th:opt2}, and Lemma \ref{optimal-imp}.
It is easy to verify that for $\gamma < \beta/4$, the coefficient is
less than $(\beta/2 + \sqrt{1-\beta})^2$ which is $<1$ (for $\beta>0$).
\end{proof}

\begin{remark}
We emphasize once again that the upper bound \eqref{up2} and Corollary \ref{thm:hilbertian} both hold for Jensen-Shannon, Hellinger, and total variation divergences among others.
In particular they can be applied to the original GAN algorithm.
\end{remark}

Conditions \ref{eq:variational-surrogate} and \ref{eq:gamma} may be compared to the ``weak learnability'' condition of AdaBoost.
As long as our weak learner is able to solve the surrogate problem \eqref{eq:what-we-can-do} of matching respectively $Q^\dagger_{\beta}$ or $Q^*_{\beta}$ accurately enough, the original objective \eqref{eq:idealObjective} is guaranteed to decrease as well.
It should be however noted that Condition~\ref{eq:gamma} with $\gamma<\beta/4$ is perhaps too strong to call it ``weak learnability''. 
Indeed, as already mentioned before, the weight $\beta$ is expected to decrease to zero as the number of components in the mixture distribution $P_g$ increases.
This leads to $\gamma \to 0$, making it harder to meet Condition \ref{eq:gamma}.
This obstacle may be partially resolved by the fact that we will use a GAN to fit $Q$, which corresponds to a relatively rich\footnote{
The hardness of meeting Condition \ref{eq:gamma} of course largely depends on the class of models $\G$ used to fit $Q$ in \eqref{eq:what-we-can-do}.
For now we ignore this question and leave it for future research.
} class of models $\G$ in \eqref{eq:what-we-can-do}.
In other words, our weak learner is not so weak.

On the other hand, Condition \eqref{eq:variational-surrogate} of Corollary \ref{th:main2} is much milder. No matter what $\gamma \in [0,1]$ and $\beta\in(0,1]$ we choose, the new component $Q$ is guaranteed to strictly improve the objective functional.
This comes at the price of the additional Condition \eqref{eq:condition-missing-modes}, which asserts that $\beta$ should be larger than the mass of true data $P_d$ missed by the current model $P_g$.
We argue that this is a rather reasonable condition: if $P_g$ misses many modes of $P_d$ we would prefer assigning a relatively large weight $\beta$ to the new component $Q$.


\section{AdaGAN}
\label{sec:adaGAN}

In this section we provide a more detailed description of Algorithm \ref{alg:AdaGAN}
from Section \ref{sec:intro-our-approach}, in
particular how to reweight the training examples for the next iteration and how to choose the mixture weights.

In a nutshell, at each iteration we want to add a new component $Q$
to the current mixture $P_g$ with weight~$\beta$,
to create a mixture with distribution $(1-\beta) P_g + \beta Q$.
This component $Q$ should approach an ``optimal target'' $Q^*_{\beta}$ and we
know from Theorem \ref{main} that:
\[
d Q^*_{\beta} = \frac{dP_d}{\beta}\left(\lambda^* - (1-\beta) \frac{dP_g}{dP_d} \right)_+\,.
\]
Computing this distribution requires to know the density ratio $dP_g/dP_d$,
which is not directly accessible, but it can be estimated using the idea of adversarial training.
Indeed, we can train a discriminator $D$ to distinguish between samples
from $P_d$ and $P_g$. 
It is known that for an arbitrary $f$-divergence, there exists a corresponding function $h$ (see \cite{nowozin2016f}) such that the values of the optimal discriminator $D$
are related to the density ratio in the following way:
\begin{equation}
 \label{eq:prob-ratio-and-d}
\frac{dP_g}{dP_d}(X) = h\bigl(D(X)\bigr).
\end{equation}
In particular, for the Jensen-Shannon divergence, used by the original GAN algorithm, it holds that $h\bigl(D(X)\bigr)=\frac{1-D(X)}{D(X)}$.
So in this case for the optimal discriminator we have
\[
d Q^*_{\beta} = \frac{dP_d}{\beta}\bigl(\lambda^* - (1-\beta) h(D) \bigr)_+\,,
\]
which can be viewed as a reweighted version of the original data
distribution $P_d$.

In particular, when we compute $d Q^*_{\beta}$ on the training sample
$S_N = (X_1,\dots,X_N)$, each example $X_i$ has the following weight:
\begin{equation}
 \label{eq:new-weights}
 w_i=\frac{p_i}{\beta}\bigl(\lambda^* - (1-\beta) h(d_i) \bigr)_+
\end{equation}
with $p_i=dP_d(X_i)$ and $d_i=D(X_i)$.
In practice, we use the empirical distribution over the training sample which
means we set $p_i=1/N$.

\subsection{How to compute $\lambda^*$ of Theorem \ref{main}}
\label{sec:lambda-compute}
Next we derive an algorithm to determine $\lambda^*$.
We need to find a value of $\lambda^*$ such that the weights $w_i$ in \eqref{eq:new-weights} are normalized, i.e.:
\[
	\sum_i w_i = \sum_{i \in \I(\lambda^*)} \frac{p_i}{\beta}\bigl(\lambda^* - (1-\beta) h(d_i) \bigr) = 1\ ,
\]
where $\I(\lambda) := \{ i : \lambda > (1-\beta) h(d_i) \}$.
This in turn yields:
\begin{equation}\label{lambdaValue}
	\lambda^* = \frac{\beta}{\sum_{i \in \I(\lambda^*)} p_i} \left (1 + \frac{(1-\beta)}{\beta} \sum_{i \in \I(\lambda^*)} p_i h(d_i) \right ) \ .
\end{equation}
Now, to compute the r.h.s., we need to know $\I(\lambda^*)$. To do so, we sort the values $h(d_i)$ in increasing order: $h(d_1) \leq h(d_2) \leq \ldots \leq h(d_N)$.
Then $\I(\lambda^*)$ is simply a set consisting of the first $k$ values, where we have to determine $k$. Thus, it suffices to test successively all positive integers $k$ until the $\lambda$ given by Equation~\eqref{lambdaValue} verifies:
\[
	(1-\beta) h(d_k) < \lambda \leq (1-\beta) h(d_{k+1}) \ .
\]
This procedure is guaranteed to converge, because by Theorem~\ref{main}, we know that $\lambda^*$ exists, and it satisfies~\eqref{lambdaValue}.
In summary, $\lambda^*$ can be determined by Algorithm~\ref{alg:algoLambda}.

\begin{algorithm}
	Sort the values $h(d_i)$ in increasing order \;
	Initialize $\lambda \leftarrow \frac{\beta}{p_1} \left (1 + \frac{1-\beta}{\beta} p_1 h(d_1) \right )$ and $k \leftarrow 1$ \;
	\While{$(1-\beta) h(d_k) \geq \lambda$}{
	$k \leftarrow k + 1$\;
	$\lambda \leftarrow \frac{\beta}{\sum_{i=1}^k p_i} \left (1 + \frac{(1-\beta)}{\beta} \sum_{i=1}^k p_i h(d_i) \right )$}
	 \caption{Determining $\lambda^*$}
	 \label{alg:algoLambda}
\end{algorithm}

%

\subsection{How to choose a mixture weight $\beta$}
\label{sec:choosing-beta}
While for every $\beta$ there is an optimal reweighting scheme, the weights from \eqref{eq:new-weights} depend on $\beta$.
In particular, if $\beta$ is large enough to verify $dP_d(x) \lambda^* - (1-\beta) dP_g(x) \geq 0$ for all $x$,
the optimal component $Q^*_{\beta}$ satisfies $(1-\beta) P_g + \beta Q^*_{\beta}=P_d$, as proved in Lemma \ref{le:conv}. 
In other words, in this case we exactly match the data distribution $P_d$, assuming the GAN can approximate the target $Q^*_{\beta}$ perfectly.
This criterion alone would lead to choosing $\beta=1$. However in practice we know
we can't get a generator that produces exactly the target distribution $Q^*_{\beta}$.
We thus propose a few heuristics one can follow to choose $\beta$:
\begin{itemize}
\item Any fixed constant value $\beta$ for all iterations.
\item All generators to be combined with equal weights in the final mixture model. This corresponds to setting $\beta_t=\frac{1}{t}$, where $t$ is the iteration.
\item Instead of choosing directly a value for $\beta$ one could
pick a ratio $0 < r < 1$ of examples which should have a weight $w_i>0$. Given such
an $r$, there is a unique value of $\beta$ $(\beta_r$) resulting in $w_i>0$ for exactly $N\cdot r$
training examples.
Such a value $\beta_r$ can be determined by binary search over $\beta$ in Algorithm~\ref{alg:algoLambda}.
Possible choices for $r$ include:
\begin{itemize}
\item $r$ constant, chosen experimentally.
\item $r$ decreasing with the number of iterations, e.g., $r = c_1 e^{-c_2 t}$ for any positive constants $c_1, c_2$.
\end{itemize}
\item Alternatively, one can set a particular threshold for the density ratio
estimate $h(D)$, compute the fraction~$r$ of training examples that
have a value above that threshold and derive $\beta$ from this ratio $r$ (as
above).
Indeed, when $h(D)$ is large, that means that the generator does
not generate enough examples in that region, and the next iteration should be encouraged
to generate more there.
\end{itemize}


\begin{algorithm}[H]\captionsetup{labelfont={sc,bf}, labelsep=newline}
\caption{AdaGAN, a meta-algorithm to construct a ``strong'' mixture of $T$ individual GANs, trained sequentially.
The mixture weight schedule ChooseMixtureWeight should be provided by the user (see \ref{sec:choosing-beta}). This is an instance
of the high level Algorithm \ref{alg:AdaGAN}, instantiating UpdateTrainingWeights.}
\label{algoAdaGANComplete}
\begin{spacing}{1.4}
\renewcommand{\algorithmicrequire}{\textbf{Input:}}
\renewcommand{\algorithmicensure}{\textbf{Output:}}
  \begin{algorithmic}
  \REQUIRE{Training sample $S_N:=\{X_1,\dots,X_N\}$.}
  \ENSURE{Mixture generative model $G=G_T$.}
  \STATE{
  Train vanilla GAN:
  $G_1 = \mathrm{GAN}(S_N)$
  }
  \FOR{$t=2,\dots,T$}
  \STATE{
  \emph{\#Choose a mixture weight for the next component}

  $\beta_t = \mathrm{ChooseMixtureWeight}(t)$

  \emph{\#Compute the new weights of the training examples (UpdateTrainingWeights)}

  \emph{\#Compute the discriminator between the original (unweighted) data and the current mixture $G_{t-1}$}

  $D \leftarrow DGAN(S_N, G_{t-1})$\;
  
  \emph{\#Compute $\lambda^*$ using Algorithm \ref{alg:algoLambda}}
  
  $\lambda^* \leftarrow \lambda(\beta_t, D)$
  
  \emph{\#Compute the new weight for each example}
  
  \FOR{$i=1,\dots,N$}
  \STATE{
    $W^i_t = \frac{1}{N \beta_t}\left(\lambda^* - (1-\beta_t) h(D(X_i)) \right)_+$
  }
  \ENDFOR

  \emph{\#Train $t$-th ``weak'' component generator $G^c_t$}
  
  $G^c_t = \mathrm{GAN}(S_N, W_t)$

  \emph{\#Update the overall generative model}
  
  \emph{\#Notation below means forming a mixture of $G_{t-1}$ and $G^c_t$.}
  
  $G_t = (1-\beta_t) G_{t-1} + \beta_t G^c_t$
  }
  \ENDFOR
  \end{algorithmic}
\end{spacing}
\end{algorithm}

\subsection{Complete algorithm}
\label{sec:AdaGanComplete}
Now we have all the necessary components to introduce the complete AdaGAN meta-algorithm. 
The algorithm uses any given GAN implementation (which can be the original one of Goodfellow et al.\:\cite{goodfellow2014generative} or any later modifications) as a building block. 
Accordingly, $G^c \leftarrow GAN(S_N, W)$  returns a generator $G^c$ for a given set of examples $S_N=(X_1,\dots,X_N)$ and corresponding weights $W=(w_1,\dots,w_N)$.
Additionally, we write $D \leftarrow DGAN(S_N, G)$ to denote a procedure that returns a discriminator
from the GAN algorithm trained on a given set of true data examples $S_N$ and examples sampled from the mixture of generators $G$.
We also write $\lambda^*(\beta, D)$ to denote the optimal $\lambda^*$ given by Algorithm \ref{alg:algoLambda}.
The complete algorithm is presented in Algorithm~\ref{algoAdaGANComplete}.

%

\section{Experiments}
\label{sec:experiments}
We tested AdaGAN\footnote{
Code available online at \url{https://github.com/tolstikhin/adagan}} on toy datasets, for which we can interpret the missing modes in a clear and reproducible way, and on MNIST, which is a high-dimensional dataset. The goal of these experiments \emph{was not} to evaluate the visual quality of individual sample points, but to demonstrate
that the re-weighting scheme of AdaGAN promotes diversity and effectively covers the missing modes.

\subsection{Toy datasets}
The target distribution is defined as a mixture of normal distributions, with different variances.
The distances between the means are relatively large compared to the variances, so that each Gaussian of the mixture is ``isolated''.
We vary the number of modes to test how well each algorithm performs when there are fewer or more expected modes.

More precisely, we set $\X=\R^2$, each Gaussian component is isotropic, and their centers are sampled uniformly in a square. That particular random seed is fixed
for all experiments, which means that for a given number of modes, the target distribution is always the same. The variance parameter is the same for each component,
and is decreasing with the number of modes, so that the modes stay apart from each other.

This target density is very easy to learn, using a mixture of Gaussians model, and for example the EM algorithm \cite{DLR77}.
If applied to the situation where the generator is producing single Gaussians (i.e.\:$P_Z$ is a standard Gaussian and $G$ is a linear function), then AdaGAN produces a mixture of Gaussians, however it does so incrementally unlike EM, which keeps a fixed number of components.
In any way AdaGAN was not tailored for this particular case and we use the Gaussian mixture model simply as a toy example to illustrate the missing modes problem.

\subsubsection{Algorithms}
We compare different meta-algorithms based on GAN, and the baseline GAN algorithm. All the meta-algorithms use the same implementation of the underlying GAN procedure. 
In all cases, the generator uses latent space $\Z=\R^5$, and two ReLU hidden layers, of size 10 and 5 respectively. The corresponding discriminator has two ReLU hidden layers of size 20 and 10
respectively. We use 64k training examples, and 15 epochs, which is enough compared to the small scale of the problem, and all networks converge properly and overfitting is never an
issue. Despite the simplicity of the problem, there are already differences between the different approaches.

We compare the following algorithms:
\begin{itemize}
\item The baseline GAN algorithm, called \textbf{Vanilla GAN} in the results.
\item The best model out of $T$ runs of GAN, that is: run $T$ GAN instances independently, then take the run that performs best on a validation set.
This gives an additional baseline with
similar computational complexity as the ensemble approaches.
Note that the selection of the best run is done on the reported target metric (see below), rather than
on the internal metric. 
As a result this baseline is slightly overestimated. This procedure is called \textbf{Best of T} in the results.
\item A mixture of $T$ GAN generators, trained independently, and combined
with equal weights (the ``bagging'' approach). This procedure is called \textbf{Ensemble} in the results.
\item A mixture of GAN generators, trained sequentially with different choices of
data reweighting:
\begin{itemize}
\item The AdaGAN algorithm (Algorithm \ref{alg:AdaGAN}), for $\beta=1/t$, i.e. each component will have the same weight in the resulting mixture (see \autoref{sec:choosing-beta}).
This procedure is called \textbf{Boosted} in the results.
\item The AdaGAN algorithm (Algorithm \ref{alg:AdaGAN}), for a constant $\beta$, exploring several values.
This procedure is called for example \textbf{Beta0.3} for $\beta=0.3$ in the results.
\item Reweighting similar to ``Cascade GAN'' from \cite{wang2016ensembles}, i.e.
keeping the top $r$ fraction of examples, based on the discriminator corresponding to the \emph{previous} generator.
This procedure is called for example \textbf{TopKLast0.3} for $r=0.3$.
\item Keep the top $r$ fraction of examples, based on the discriminator corresponding to
\emph{the mixture of all previous} generators. 
This procedure is called for example \textbf{TopK0.3} for $r=0.3$.
\end{itemize}
\end{itemize}

\subsubsection{Metrics}
\label{sec:metrics}
To evaluate how well the generated distribution matches the target distribution,
we use a \emph{coverage} metric~$C$. We compute the probability mass of the true data ``covered''
by the model distribution $P_{model}$. More precisely, we compute ${C:=P_d( dP_{model} > t)}$ with $t$ such that
$P_{model}(dP_{model} > t) = 0.95$. This metric is more interpretable than the likelihood,
making it easier to assess the difference in performance of the algorithms.
To approximate the density of $P_{model}$ we use a kernel density estimation
method, where the bandwidth is chosen by cross validation. Note that we could also use the discriminator $D$
to approximate the coverage as well, using the relation from \eqref{eq:prob-ratio-and-d}.

Another metric is the likelihood of the true data under the generated distribution.
More precisely, we compute $L:=\frac{1}{N} \sum_i \log P_{model}(x_i)$, on a sample of $N$
examples from the data.
Note that \cite{1611.04273} proposes a more general and elegant approach (but
less straightforward to implement) to have an objective measure of GAN. On the simple
problems we tackle here, we can precisely estimate the likelihood.

In the main results we report the metric $C$ and in Appendix \ref{sec:additional_experiments} we report
both $L$ and $C$.
For a given metric, we repeat the run 35 times with the same parameters (but different random seeds).
For each run, the learning rate is optimized using a grid search on a validation set.
We report the median over those multiple runs, and the interval corresponding to the 5\% and 95\% percentiles. Note
this is not a \textit{confidence interval} of the median, which would shrink to a singleton with an infinite number of runs.
Instead, this gives a measure of the stability of each algorithm.
The optimizer is a simple SGD: Adam was also tried but gave slightly less stable results.

\label{experiments}

\subsubsection{Results}
With the vanilla GAN algorithm, we observe that not all the modes are covered
(see Figure \ref{fig:missing_modes} for an illustration).
Different modes (and even different number of modes) are possibly covered at each restart of the algorithm, so restarting the
algorithm with different random seeds and taking the best (``best of $T$'') can improve the results.

Figure \ref{fig:coverage_per_iteration} summarizes the performance of the main algorithms on the $C$ metric, as a function of the
number of iterations $T$. Table \ref{table:simple_comparison} gives more detailed results, varying the number of modes for the target distribution.
Appendix \ref{sec:additional_experiments} contains details on variants for the reweighting heuristics as well as results for the $L$ metric.

As expected, both the ensemble and the boosting approaches significantly outperform the vanilla GAN and the ``best of $T$'' algorithm.
Interestingly, the improvements are significant even after just one or two additional iterations ($T=2$ or $T=3$).
The boosted approach converges much faster. In addition, the variance is much lower, improving the likelihood
that a given run gives good results. On this setup, the vanilla GAN approach has a significant
number of catastrophic failures (visible in the lower bound of the interval).

Empirical results on combining AdaGAN meta-algorithm with the unrolled GANs \cite{MPPS2017} are available in Appendix~\ref{app:experiments}.

\begin{figure}[h!]
    \centering
    \includegraphics[width=0.3\linewidth]{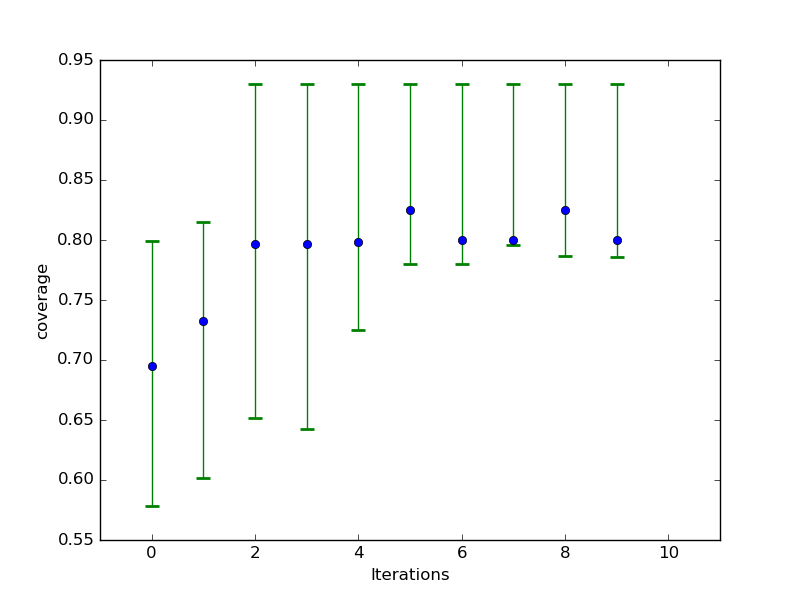}
    \includegraphics[width=0.3\linewidth]{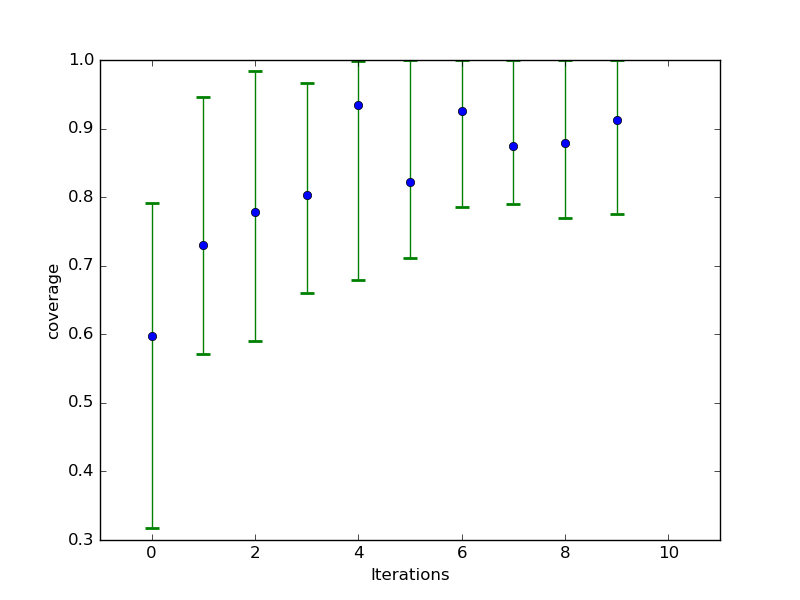}
    \includegraphics[width=0.3\linewidth]{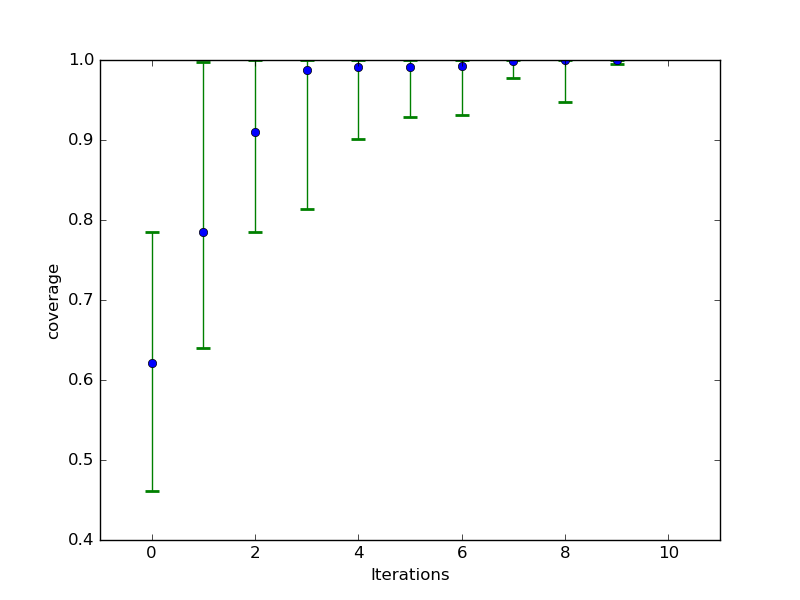}
    \caption{Coverage $C$ of the true data by the model distribution $P_{model}^T$, as a function of  iterations $T$. Experiments correspond to the data distribution with 5 modes.
	     Each blue point is the median over 35 runs.
	     Green intervals are defined by the 5\% and 95\% percentiles (see Section \ref{sec:metrics}).
             Iteration 0 is equivalent to one vanilla GAN.
             The left plot corresponds to taking the best generator out of $T$ runs.
             The middle plot corresponds to the ``ensemble GAN'', simply taking a uniform mixture
             of $T$ independently trained GAN generators. 
             The right plot corresponds to our boosting approach (AdaGAN),
             carefully reweighting the examples based on the previous generators, with $\beta_t=1/t$.
             Both the ensemble and boosting approaches significantly outperform the vanilla approach with few
             additional iterations. They also outperform taking the best out of $T$ runs.
             The boosting outperforms all other approaches. 
             For AdaGAN the variance of the performance is 
             also significantly decreased.
    }
    \label{fig:coverage_per_iteration}
\end{figure}

\begin{table}[h!]
\begin{center}
\ \hspace*{-1.5cm}
{\renewcommand{\arraystretch}{1.7}
\begin{tabular}{ | l | c|c|c|c|c|c | }
  \hline
     &   $Modes: 1$  & $Modes: 2$  & $Modes: 3$  & $Modes: 5$  & $Modes: 7$  & $Modes: 10$  \\ \hline
 Vanilla  &  \begin{minipage}{2cm}$0.97$ \footnotesize{$(0.9;1.0)$}\end{minipage}  &  \begin{minipage}{2cm}$0.88$ \footnotesize{$(0.4;1.0)$}\end{minipage}  &  \begin{minipage}{2cm}$0.63$ \footnotesize{$(0.5;1.0)$}\end{minipage}  &  \begin{minipage}{2cm}$0.72$ \footnotesize{$(0.5;0.8)$}\end{minipage}  &  \begin{minipage}{2cm}$0.58$ \footnotesize{$(0.4;0.8)$}\end{minipage}  &  \begin{minipage}{2cm}$0.59$ \footnotesize{$(0.2;0.7)$}\end{minipage}  \\ \hline
 Best of T (T=3)  &  \begin{minipage}{2cm}$0.99$ \footnotesize{$(1.0;1.0)$}\end{minipage}  &  \begin{minipage}{2cm}$0.96$ \footnotesize{$(0.9;1.0)$}\end{minipage}  &  \begin{minipage}{2cm}$0.91$ \footnotesize{$(0.7;1.0)$}\end{minipage}  &  \begin{minipage}{2cm}$0.80$ \footnotesize{$(0.7;0.9)$}\end{minipage}  &  \begin{minipage}{2cm}$0.84$ \footnotesize{$(0.7;0.9)$}\end{minipage}  &  \begin{minipage}{2cm}$0.70$ \footnotesize{$(0.6;0.8)$}\end{minipage}  \\ \hline
 Best of T (T=10)  &  \begin{minipage}{2cm}$0.99$ \footnotesize{$(1.0;1.0)$}\end{minipage}  &  \begin{minipage}{2cm}$0.99$ \footnotesize{$(1.0;1.0)$}\end{minipage}  &  \begin{minipage}{2cm}$0.98$ \footnotesize{$(0.8;1.0)$}\end{minipage}  &  \begin{minipage}{2cm}$0.80$ \footnotesize{$(0.8;0.9)$}\end{minipage}  &  \begin{minipage}{2cm}$0.87$ \footnotesize{$(0.8;0.9)$}\end{minipage}  &  \begin{minipage}{2cm}$0.71$ \footnotesize{$(0.7;0.8)$}\end{minipage}  \\ \hline
 Ensemble (T=3)  &  \begin{minipage}{2cm}$0.99$ \footnotesize{$(1.0;1.0)$}\end{minipage}  &  \begin{minipage}{2cm}$0.98$ \footnotesize{$(0.9;1.0)$}\end{minipage}  &  \begin{minipage}{2cm}$0.93$ \footnotesize{$(0.8;1.0)$}\end{minipage}  &  \begin{minipage}{2cm}$0.78$ \footnotesize{$(0.6;1.0)$}\end{minipage}  &  \begin{minipage}{2cm}$0.85$ \footnotesize{$(0.6;1.0)$}\end{minipage}  &  \begin{minipage}{2cm}$0.80$ \footnotesize{$(0.6;1.0)$}\end{minipage}  \\ \hline
 Ensemble (T=10)  &  \begin{minipage}{2cm}$1.00$ \footnotesize{$(1.0;1.0)$}\end{minipage}  &  \begin{minipage}{2cm}$0.99$ \footnotesize{$(1.0;1.0)$}\end{minipage}  &  \begin{minipage}{2cm}$1.00$ \footnotesize{$(1.0;1.0)$}\end{minipage}  &  \begin{minipage}{2cm}$0.91$ \footnotesize{$(0.8;1.0)$}\end{minipage}  &  \begin{minipage}{2cm}$0.88$ \footnotesize{$(0.8;1.0)$}\end{minipage}  &  \begin{minipage}{2cm}$0.89$ \footnotesize{$(0.7;1.0)$}\end{minipage}  \\ \hline
 TopKLast0.5 (T=3)  &  \begin{minipage}{2cm}$0.98$ \footnotesize{$(0.9;1.0)$}\end{minipage}  &  \begin{minipage}{2cm}$0.98$ \footnotesize{$(0.9;1.0)$}\end{minipage}  &  \begin{minipage}{2cm}$0.95$ \footnotesize{$(0.9;1.0)$}\end{minipage}  &  \begin{minipage}{2cm}$0.95$ \footnotesize{$(0.8;1.0)$}\end{minipage}  &  \begin{minipage}{2cm}$0.86$ \footnotesize{$(0.7;1.0)$}\end{minipage}  &  \begin{minipage}{2cm}$0.86$ \footnotesize{$(0.6;0.9)$}\end{minipage}  \\ \hline
 TopKLast0.5 (T=10)  &  \begin{minipage}{2cm}$0.99$ \footnotesize{$(1.0;1.0)$}\end{minipage}  &  \begin{minipage}{2cm}$0.98$ \footnotesize{$(0.9;1.0)$}\end{minipage}  &  \begin{minipage}{2cm}$0.98$ \footnotesize{$(1.0;1.0)$}\end{minipage}  &  \begin{minipage}{2cm}$0.99$ \footnotesize{$(0.8;1.0)$}\end{minipage}  &  \begin{minipage}{2cm}$0.99$ \footnotesize{$(0.8;1.0)$}\end{minipage}  &  \begin{minipage}{2cm}$1.00$ \footnotesize{$(0.8;1.0)$}\end{minipage}  \\ \hline
 Boosted (T=3)  &  \begin{minipage}{2cm}$0.99$ \footnotesize{$(1.0;1.0)$}\end{minipage}  &  \begin{minipage}{2cm}$0.99$ \footnotesize{$(0.9;1.0)$}\end{minipage}  &  \begin{minipage}{2cm}$0.98$ \footnotesize{$(0.9;1.0)$}\end{minipage}  &  \begin{minipage}{2cm}$0.91$ \footnotesize{$(0.8;1.0)$}\end{minipage}  &  \begin{minipage}{2cm}$0.91$ \footnotesize{$(0.8;1.0)$}\end{minipage}  &  \begin{minipage}{2cm}$0.86$ \footnotesize{$(0.7;1.0)$}\end{minipage}  \\ \hline
 Boosted (T=10)  &  \begin{minipage}{2cm}$1.00$ \footnotesize{$(1.0;1.0)$}\end{minipage}  &  \begin{minipage}{2cm}$1.00$ \footnotesize{$(1.0;1.0)$}\end{minipage}  &  \begin{minipage}{2cm}$1.00$ \footnotesize{$(1.0;1.0)$}\end{minipage}  &  \begin{minipage}{2cm}$1.00$ \footnotesize{$(1.0;1.0)$}\end{minipage}  &  \begin{minipage}{2cm}$1.00$ \footnotesize{$(1.0;1.0)$}\end{minipage}  &  \begin{minipage}{2cm}$1.00$ \footnotesize{$(1.0;1.0)$}\end{minipage}  \\ \hline
  \end{tabular}}
\end{center}

 \caption{Performance of the different algorithms on varying number of mixtures of Gaussians.
   The reported score is the coverage $C$, probability mass of $P_d$ covered   
   by the $5th$ percentile of $P_g$ defined in Section \ref{sec:metrics}. See Table \ref{table:simple_comparison_complete} for more metrics.
   The reported scores are the median and interval defined by the 5\% and 95\% percentile (in parenthesis) (see Section \ref{sec:metrics}), over 35 runs for each setting.
   Note that the $95\%$ interval is not the usual confidence interval measuring the
   variance of the experiment itself, but rather measures the stability of the
   different algorithms (would remain even if each experiment was run an infinite
   number of times).
   Both the ensemble and the boosting approaches significantly outperform the vanilla GAN even with just three iterations (i.e. just
   two additional components). The boosting approach converges faster to the optimal coverage and with smaller variance.
 }

\label{table:simple_comparison}

\end{table}


\subsection{MNIST and MNIST3}
We ran experiments both on the original MNIST and on the 3-digit MNIST (MNIST3) \cite{che2016mode, MPPS2017} dataset, obtained by concatenating 3 randomly chosen MNIST images to form a 3-digit number between 0 and 999. 
According to \cite{che2016mode, MPPS2017}, MNIST contains 10 modes, while MNIST3 contains 1000 modes, and these modes can be detected using the pre-trained MNIST classifier.
We combined AdaGAN both with simple MLP GANs and DCGANs \cite{RMC16}.
We used $T\in\{5,10\}$, tried models of various sizes and performed a reasonable amount of hyperparameter search. 
For the details we refer to Appendix \ref{app:experiments-mnist}.

Similarly to \cite[Sec 3.3.1]{MPPS2017} we failed to reproduce the missing modes problem for MNIST3 reported in \cite{che2016mode} and found that simple GAN architectures are capable of generating all 1000 numbers.
The authors of~\cite{MPPS2017} proposed to artificially introduce the missing modes again by limiting the generators' flexibility. 
In our experiments, GANs trained with the architectures reported in \cite{MPPS2017} were often generating poorly looking digits.
As a result, the pre-trained MNIST classifier was outputting random labels, which again led to full coverage of the 1000 numbers. We tried to threshold the confidence of the pre-trained classifier, but decided that this metric was too ad-hoc.

\begin{wrapfigure}[19]{r}{0.45\textwidth}
	\vspace{-0pt}
	\includegraphics[width=0.45\textwidth]{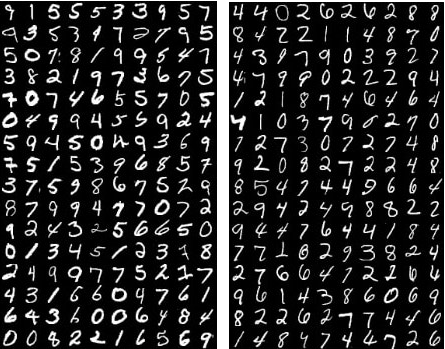}
	\caption{Digits from the MNIST dataset corresponding to the smallest ({\bf left}) and largest ({\bf right}) weights, obtained by the AdaGAN procedure (see Section \ref{sec:adaGAN}) in one of the runs. Bold digits (left) are already covered and next GAN will concentrate on thin (right) digits.}
	\label{fig:coverage_per_iteration}
\end{wrapfigure}


For MNIST we noticed that the re-weighted distribution was often concentrating its mass on digits having very specific strokes: on different rounds it could highlight thick, thin, vertical, or diagonal digits, indicating that these traits were underrepresented in the generated samples (see Figure~\ref{fig:coverage_per_iteration}). This suggests that AdaGAN does a reasonable job at picking up different modes of the dataset, but also that there are more than 10 modes in MNIST (and more than 1000 in MNIST3). 
It is not clear how to evaluate the quality of generative models in this context.

We also tried to use the ``inversion'' metric discussed in Section 3.4.1 of \cite{MPPS2017}. For MNIST3 we noticed that a single GAN was 
capable of reconstructing most of the training points \emph{very} accurately both visually and in the $\ell_2$-reconstruction sense.

\section{Conclusion}

We presented an incremental procedure for constructing an additive
mixture of generative models by minimizing an $f$-divergence criterion.
Based on this, we derived a boosting-style algorithm for GANs, which we call \emph{AdaGAN}.
By incrementally adding new generators into a mixture through the optimization of a GAN criterion on a reweighted data, this algorithm is able to progressively cover all the modes
of the true data distribution. This addresses one of the main practical issues
of training GANs.

We also presented a theoretical analysis of the convergence of this incremental
procedure and showed conditions under which the mixture converges to the true
distribution either exponentially or in a finite number of steps.

Our preliminary experiments (on toy data) show that this algorithm is effectively
addressing the missing modes problem and allows to robustly produce a mixture which 
covers all modes of the data.

However, since the generative model that we obtain is not a single 
neural network but a mixture of such networks, the corresponding latent
representation no longer has a smooth structure. This can be seen
as a disadvantage compared to standard GAN where one can perform smooth
interpolation in latent space. On the other hand it also allows to
have a partitioned latent representation where one component is discrete.
Future work will explore the possibility of leveraging this structure
to model discrete aspects of the dataset, such as the class in object
recognition datasets in a similar spirit to \cite{chen2016infogan}.

\bibliographystyle{unsrt}
\bibliography{adagan}

\newpage

\appendix

\section{Further details on toy experiments}
\label{app:experiments}

To illustrate the 'meta-algorithm aspect' of AdaGAN, we also performed experiments with
 an unrolled GAN \cite{MPPS2017} instead of a GAN as the base generator. We trained the GANs both with the Jensen-Shannon objective \eqref{eq:gan}, and with its modified version proposed in \cite{goodfellow2014generative}
(and often considered as the baseline GAN), where $\log \paren*{1-D(G(Z))}$ is replaced by $-\log \paren*{D(G(Z))}$. We use the same network architecture as in the other toy experiments. Figure~\ref{fig:unrolled} illustrates our results.
We find that AdaGAN works with all underlying GAN algorithms. Note that, where the usual GAN updates the generator and the discriminator once, an unrolled GAN with 5 unrolling steps updates the generator once and the discriminator 1 + 5, i.e. 6 times (and then rolls back 5 steps). Thus, in terms of computation time, training 1 single unrolled GAN roughly corresponds to doing 3 steps of AdaGAN with a usual GAN. In that sense, Figure~\ref{fig:unrolled} shows that AdaGAN (with a usual GAN) significantly outperforms a single unrolled GAN. Additionally, we note that using the Jensen-Shannon objective (rather than the modified version) seems to have some mode-regularizing effect. Surprisingly, using unrolling steps makes no significant difference.
\begin{figure}[hb]
    \centering
    \includegraphics[width=\linewidth]{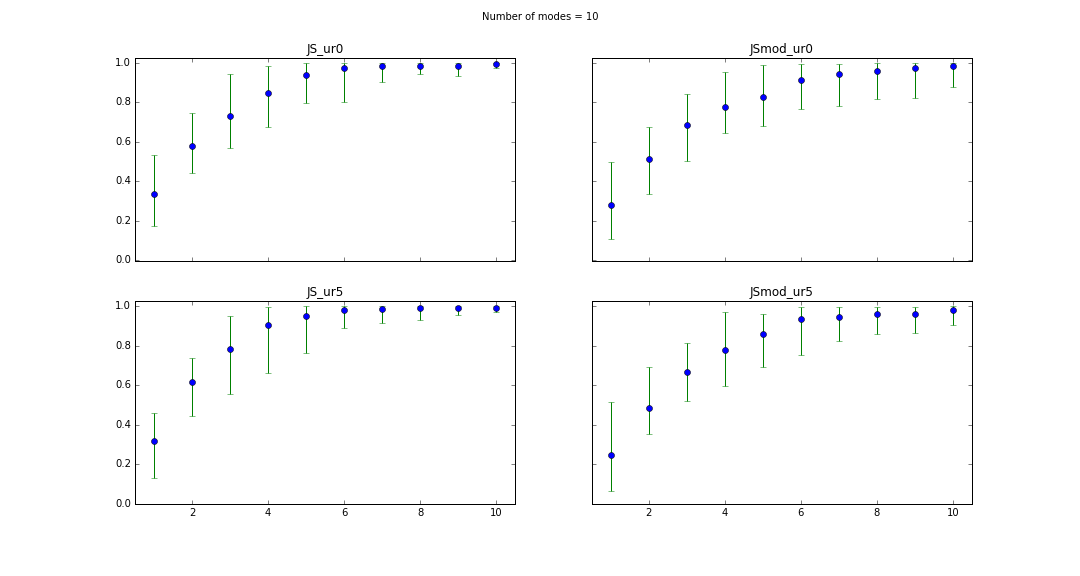}
    \caption{Comparison of AdaGAN ran with a GAN (top row) and with an unrolled GAN \cite{MPPS2017} (bottom).
    	     Coverage $C$ of the true data by the model distribution $P_{model}^T$, as a function of  iterations $T$.
    	     Experiments are similar to those of Figure~\ref{fig:coverage_per_iteration}, but with 10 modes. Top and bottom
	     rows correspond to the usual and the unrolled GAN (with 5 unrolling steps) respectively, trained with the
	     Jensen-Shannon objective~\eqref{eq:gan} on the left, and with the modified objective originally proposed
	     by \cite{goodfellow2014generative} on the right. In terms of computation time, one step of AdaGAN with unrolled GAN
	     corresponds to roughly 3 steps of AdaGAN with a usual GAN.
	     On all images $T=1$ corresponds to vanilla unrolled GAN.
	     \label{fig:unrolled}
    }
\end{figure}

\section{Further details on MNIST/MNIST3 experiments}
\label{app:experiments-mnist}
\paragraph{GAN Architecture}
We ran AdaGAN on MNIST (28x28 pixel images) using (de)convolutional networks with batch normalizations and leaky ReLu.
The latent space has dimension 100.
We used the following architectures: 
\begin{align*}
	\text{Generator:} \quad \text{100 x 1 x 1 } &\rightarrow \text{ fully connected $\rightarrow$ 7 x 7 x 16 $\rightarrow$ deconv $\rightarrow$ 14 x 14 x 8 }\rightarrow \\
		 &\rightarrow \text{ deconv $\rightarrow$ 28 x 28 x 4 $\rightarrow$ deconv $\rightarrow$ 28 x 28 x 1} \\
	\text{Discriminator:} \quad \text{28 x 28 x 1 } & \rightarrow\text{ conv $\rightarrow$ 14 x 14 x 16 $\rightarrow$ conv $\rightarrow$ 7 x 7 x 32 }\rightarrow \\
		 &\rightarrow\text{ fully connected $\rightarrow$ 1} \,
\end{align*}
where each arrow consists of a leaky ReLu (with 0.3 leak) followed by a batch normalization, conv and deconv are convolutions and
transposed convolutions with 5x5 filters, and fully connected are linear layers with bias. The distribution over $\Z$ is uniform over the unit box.
We use the Adam optimizer with $\beta_1 = 0.5$, with 2 G steps for 1 D step and learning rates 0.005 for G, 0.001 for D, and 0.0001 for the 
classifier C that does the reweighting of digits. We optimized D and G over 200 epochs and C over 5 epochs, using the original Jensen-Shannon
objective \eqref{eq:gan}, without the log trick, with no unrolling and with minibatches of size 128.

\paragraph{Empirical observations}
Although we could not find any appropriate metric to measure the increase of diversity promoted by AdaGAN, we observed that the 
re-weighting scheme indeed focuses on digits with very specific strokes. In Figure~\ref{fig:mnist} for example, we see that after one AdaGAN step,
the generator produces overly thick digits (top left image). Thus AdaGAN puts small weights on the thick digits of the dataset (bottom left)
and high weights on the thin ones (bottom right). After the next step, the new GAN produces both thick and thin digits.

    \begin{figure}[ht]
        \centering
        \begin{minipage}{0.475\linewidth}
            \includegraphics[width=\linewidth]{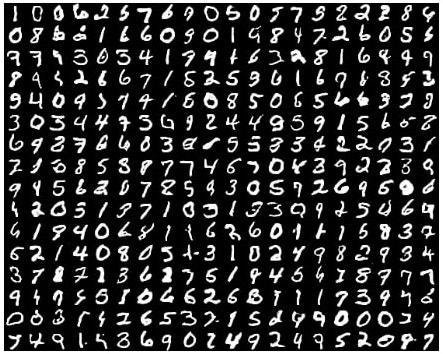}%
        \end{minipage}
        \begin{minipage}{0.475\linewidth}
            \includegraphics[width=\linewidth]{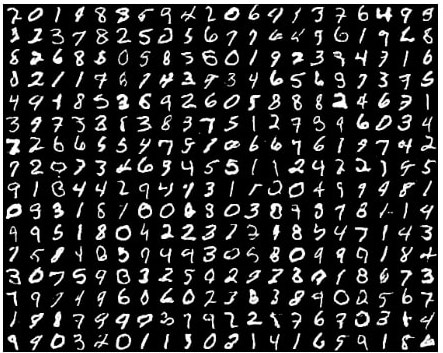}%
        \end{minipage}
        
        \begin{minipage}{0.475\linewidth}
            \includegraphics[width=\linewidth]{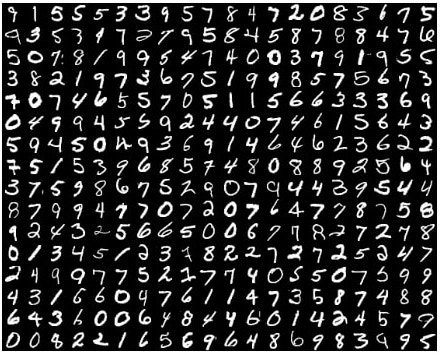}%
        \end{minipage}
        \begin{minipage}{0.475\linewidth}
            \includegraphics[width=\linewidth]{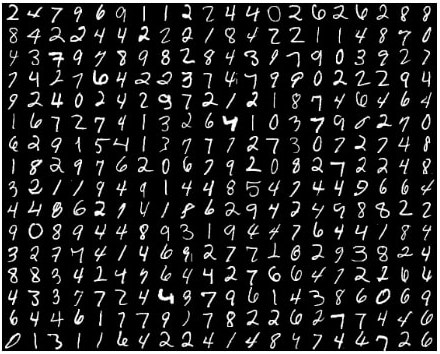}%
        \end{minipage}
        \caption{AdaGAN on MNIST. 
        		Bottom row are true MNIST digits with smallest (left) and highest (right) weights after re-weighting at the
		end of the first AdaGAN step. Those with small weight are thick and resemble those generated by the GAN 
		after the first AdaGAN step (top left). After training with the re-weighted dataset during the second iteration
		of AdaGAN, the new mixture produces more thin digits (top right).
		} 
        \label{fig:mnist}
    \end{figure}

\section{Proofs}
\subsection{Proof of Theorem \ref{main}}
\label{sec:proof-main}
Before proving Theorem \ref{main}, we introduce two lemmas. The first one
is about the determination of the constant $\lambda$, the second one is
about comparing the divergences of mixtures.

\begin{lemma}\label{le:lambda}
Let $P$ and $Q$ be two distributions, $\gamma\in [0,1]$ and $\lambda \in \R$.
The function
\[
g(\lambda)\bydef \int \paren*{\lambda - \gamma \frac{dQ}{dP}}_+dP
\]
is nonnegative, convex, nondecreasing, satisfies $g(\lambda) \leq \lambda$, and its right derivative is given by
\[
{g'_+(\lambda) = P\left(\lambda \cdot dP \ge \gamma \cdot dQ\right)}.
\] 
The equation
\[
g(\lambda) = 1-\gamma
\]
has a solution $\lambda^*$ (unique when $\gamma<1$)
with $\lambda^*\in [1-\gamma,1]$.
Finally, if $P(dQ = 0) \geq \delta$ for a strictly positive constant $\delta$ then $\lambda^*\leq (1-\gamma)\delta^{-1}$.
\end{lemma}
\begin{proof}
The convexity of $g$ follows immediately from the convexity of $x\mapsto (x)_+$
and the linearity of the integral. Similarly, since $x\mapsto (x)_+$ is
non-decreasing, $g$ is non-decreasing.

We define the set $\I(\lambda)$ as follows:
\[
\I(\lambda)\bydef \set*{x\in \X: \lambda \cdot dP(x) \ge \gamma \cdot dQ(x)}.
\]
Now let us consider $g(\lambda+\epsilon)-g(\lambda)$ for some small $\epsilon>0$.
This can also be written:
\begin{align*}
g(\lambda+\epsilon)-g(\lambda)
&= \int_{\I(\lambda)} \epsilon dP +
\int_{\I(\lambda+\epsilon) \backslash \I(\lambda)} (\lambda+\epsilon) dP - \int_{\I(\lambda+\epsilon) \backslash \I(\lambda)}\gamma dQ\\
&= \epsilon P(\I(\lambda)) +
\int_{\I(\lambda+\epsilon) \backslash \I(\lambda)} (\lambda+\epsilon) dP -\int_{\I(\lambda+\epsilon) \backslash \I(\lambda)} \gamma dQ.
\end{align*}
On the set $\I(\lambda+\epsilon) \backslash \I(\lambda)$, we have
\[
(\lambda+\epsilon)dP - \gamma dQ \in [0, \epsilon].
\]
So that
\[
\epsilon P(\I(\gamma)) \le g(\lambda+\epsilon)-g(\lambda) \le \epsilon
P(\I(\gamma)) + \epsilon P\bigl(\I(\lambda+\epsilon) \backslash \I(\lambda)\bigr) =
\epsilon P(\I(\lambda+\epsilon))
\]
and thus
\[
\lim_{\epsilon\to 0^+} \frac{g(\lambda+\epsilon)-g(\lambda)}{\epsilon} =
\lim_{\epsilon\to 0^+} P(\I(\lambda+\epsilon)) = P(\I(\lambda)).
\]
This gives the expression of the right derivative of $g$.
Moreover, notice that for $\lambda, \gamma > 0$
\[
g'_+(\lambda)
=
P\left(\lambda \cdot dP \ge \gamma \cdot dQ\right)
=
P\left(\frac{dQ}{dP} \leq \frac{\lambda}{\gamma}\right)
=
1 - P\left(\frac{dQ}{dP} > \frac{\lambda}{\gamma}\right)
\geq
1 - \gamma/\lambda
\]
by Markov's inequality.

It is obvious that $g(0)=0$.
By Jensen's inequality applied to the convex function $x\mapsto (x)_+$,
we have $g(\lambda)\ge \paren*{\lambda - \gamma}_+$. So $g(1)\ge 1-\gamma$.
Also, $g=0$ on $\R^-$ and $g\le \lambda$. This means $g$ is continuous on $\R$ and thus reaches
the value $1-\gamma$ on the interval $(0,1]$ which shows the existence
of $\lambda^*\in (0,1]$. To show that $\lambda^*$ is unique we notice that since $g(x)=0$
on $\R^-$, $g$ is convex and non-decreasing, $g$ cannot be constant on
an interval not containing $0$, and thus $g(x)=1-\gamma$ has a unique solution
for $\gamma<1$.

Also by convexity of $g$,
\[
g(0)-g(\lambda^*)\ge -\lambda^* g'_+(\lambda^*),
\]
which gives $\lambda^*\ge (1-\gamma)/g'_+(\lambda^*) \ge 1-\gamma$
since $g'_+\le 1$.
If $P(dQ = 0) \geq \delta > 0$ then also $g'_+(0) \geq \delta > 0$.
Using the fact that $g'_+$ is increasing we conclude that $\lambda^* \leq (1-\gamma)\delta^{-1}$.
\end{proof}

Next we introduce some simple convenience lemma for comparing convex functions
of random variables.
\begin{lemma}\label{le:imp}
Let $f$ be a convex function, $X,Y$ be real-valued random variables and
$c\in \R$ be a constant such that
\[
\e{\max(c,Y)}=\e{X+Y}.
\]
Then we have the following bound:
\begin{equation}\label{eq:max}
\e{f(\max(c,Y))} \le \e{f(X+Y)} - \e{X(f'(Y)-f'(c))_+} \le \e{f(X+Y)}.
\end{equation}
If in addition, $Y\le M$ a.s.\:for $M\ge c$, then
\begin{equation}\label{eq:max2}
\e{f(\max(c,Y))} \le f(c) + \frac{f(M)-f(c)}{M-c}\paren*{\e{X+Y}-c}.
\end{equation}
\end{lemma}
\begin{proof}
We decompose the expectation with respect to the value of the max, and use the
convexity of $f$:
\begin{eqnarray*}
f(X+Y)-f(\max(c,Y))
&=& \oo{Y\le c}\paren*{f(X+Y)-f(c)} + \oo{Y>c}\paren*{f(X+Y)-f(Y)}\\
&\ge& \oo{Y\le c}f'(c)\paren*{X+Y-c} + \oo{Y>c}Xf'(Y)\\
&=& (1-\oo{Y> c})Xf'(c) + f'(c)\paren*{Y-\max(c,Y)} + \oo{Y>c}Xf'(Y)\\
&=& f'(c)\paren*{X+Y-\max(c,Y)} + \oo{Y>c}X(f'(Y)-f'(c))\\
&=& f'(c)\paren*{X+Y-\max(c,Y)} + X(f'(Y)-f'(c))_+,
\end{eqnarray*}
where we used that $f'$ is non-decreasing in the last step.
Taking the expectation gives the first inequality.

For the second inequality, we use the convexity of $f$ on the interval $[c,M]$:
\[
f(\max(c,Y))
\le f(c) + \frac{f(M)-f(c)}{M-c}\paren*{\max(c,Y)-c}.
\]
Taking an expectation on both sides gives the second inequality.
\end{proof}

\begin{proof}[Theorem \ref{main}]
We first apply Lemma \ref{le:lambda} with $\gamma = 1-\beta$ and this proves the
existence of $\lambda^*$ in the interval $(\beta,1]$, which shows that $Q^*_{\beta}$
is indeed well-defined as a distribution.

Then we use Inequality \eqref{eq:max} of Lemma \ref{le:imp} with $X=\beta dQ/dP_d$,
$Y=(1-\beta)dP_g/dP_d$, and
$c=\lambda^*$.
We easily verify that $X+Y=((1-\beta)dP_g+\beta dQ)/dP_d$ and
$\max(c,Y)=((1-\beta)dP_g+\beta dQ_{\beta}^*)/dP_d$ and both have expectation $1$ with
respect to $P_d$.
We thus obtain for any distribution $Q$,
\[
D_f((1-\beta)P_g + \beta Q_{\beta}^* \, \| \, P_d) \le D_f((1-\beta)P_g + \beta Q \, \| \, P_d)\,.
\]
This proves the optimality of $Q^*_{\beta}$.
\end{proof}

\subsection{Proof of Theorem \ref{th:opt2}}
\label{sec:proof-opt2}

\begin{lemma}\label{le:lambda2}
Let $P$ and $Q$ be two distributions, $\gamma\in (0,1)$, and $\lambda \geq 0$.
The function
\[
h(\lambda)\bydef \int \paren*{\frac{1}{\gamma} -\lambda \frac{dQ}{dP}}_+dP
\]
is convex, non-increasing, and its right derivative is given by
$h'_+(\lambda) = -Q(1/\gamma \geq \lambda dQ(X)/dP(X))$. 
Denote $\Delta := P( dQ(X)/dP(X)= 0)$.
Then the equation
\[
h(\lambda) = \frac{1-\gamma}{\gamma}
\]
has no solutions if $\Delta > 1 - \gamma$,
has a single solution $\lambda^\dagger \geq 1$ if $\Delta < 1 - \gamma$,
and has infinitely many or no solutions when $\Delta = 1 - \gamma$.
\end{lemma}
\begin{proof}
The convexity of $h$ follows immediately from the convexity of $x\mapsto (a - x)_+$
and the linearity of the integral. Similarly, since $x\mapsto (a-x)_+$ is
non-increasing, $h$ is non-increasing as well.

We define the set $\J(\lambda)$ as follows:
\[
\J(\lambda)\bydef \left\{x\in \X\colon \frac{1}{\gamma} \geq \lambda \frac{dQ}{dP}(x)\right\}.
\]
Now let us consider $h(\lambda)-h(\lambda+\epsilon)$ for any $\epsilon>0$.
Note that $\J(\lambda+\epsilon) \subseteq \J(\lambda)$.
We can write:
\begin{align*}
h(\lambda)-h(\lambda+\epsilon)
&=
\int_{\J(\lambda)}\paren*{\frac{1}{\gamma} -\lambda \frac{dQ}{dP}}dP
-
\int_{\J(\lambda+\epsilon)}\paren*{\frac{1}{\gamma} -(\lambda+\epsilon) \frac{dQ}{dP}}dP\\
&=
\int_{\J(\lambda)\setminus \J(\lambda+\epsilon)}\paren*{\frac{1}{\gamma} -\lambda \frac{dQ}{dP}}dP
+
\int_{\J(\lambda+\epsilon)}\paren*{\epsilon \frac{dQ}{dP}}dP\\
&=
\int_{\J(\lambda)\setminus \J(\lambda+\epsilon)}\paren*{\frac{1}{\gamma} -\lambda \frac{dQ}{dP}}dP
+
\epsilon \cdot Q(\J(\lambda+\epsilon)).
\end{align*}
Note that for $x \in \J(\lambda)\setminus \J(\lambda+\epsilon)$ we have
\[
0 \leq\frac{1}{\gamma} -\lambda \frac{dQ}{dP}(x) < \epsilon \frac{dQ}{dP}(x).
\]
This gives the following:
\[
\epsilon \cdot Q(\J(\lambda+\epsilon))
\leq 
h(\lambda)-h(\lambda+\epsilon)
\leq 
\epsilon \cdot Q(\J(\lambda+\epsilon))
+
\epsilon \cdot Q(\J(\lambda)\setminus \J(\lambda+\epsilon))
=
\epsilon \cdot Q(\J(\lambda)),
\]
which shows that $h$ is continuous.
Also
\[
\lim_{\epsilon\to 0^+} \frac{h(\lambda+\epsilon)-h(\lambda)}{\epsilon} =
\lim_{\epsilon\to 0^+} -Q(\J(\lambda+\epsilon)) = -Q(\J(\lambda)).
\]

It is obvious that $h(0)=1/\gamma$ and $h\le \gamma^{-1}$ for $\lambda \geq 0$. 
By Jensen's inequality applied to the convex function $x\mapsto (a - x)_+$,
we have $h(\lambda)\geq \paren*{\gamma^{-1} - \lambda}_+$. So $h(1)\geq \gamma^{-1}-1$.
We conclude that $h$ may reach the value $(1-\gamma)/\gamma = \gamma^{-1} - 1$ only on $[1,+\infty)$.
Note that 
\[
h(\lambda) \to \frac{1}{\gamma}P\left( \frac{dQ}{dP}(X)= 0\right) = \frac{\Delta}{\gamma}\geq 0\quad\text{as}\quad \lambda \to \infty.
\]
Thus if $\Delta/\gamma > \gamma^{-1} - 1$ the equation $h(\lambda) = \gamma^{-1} - 1$ has no solutions, as $h$ is non-increasing.
If $\Delta/\gamma = \gamma^{-1} - 1$ then either $h(\lambda) > \gamma^{-1} - 1$ for all $\lambda\geq 0$ and we have no solutions or there is a finite $\lambda'\geq 1$ such that $h(\lambda') = \gamma^{-1} - 1$, which means that the equation is also satisfied by all $\lambda \geq \lambda'$, as $h$ is continuous and non-increasing.
Finally, if $\Delta/\gamma < \gamma^{-1} - 1$ then there is a unique $\lambda^\dagger$ such that $h(\lambda^\dagger) = \gamma^{-1} - 1$, which follows from the convexity of $h$.
\end{proof}

Next we introduce some simple convenience lemma for comparing convex functions
of random variables.
\begin{lemma}\label{le:imp2}
Let $f$ be a convex function, $X,Y$ be real-valued random variables such that
$X\le Y$ a.s., and $c\in \R$ be a constant such that\footnote{Generally it is not guaranteed that such a constant $c$ always exists. In this result we assume this is the case.}
\[
\e{\min(c,Y)}=\e{X}.
\]
Then we have the following lower bound:
\[
\e{f(X)-f(\min(c,Y))} \ge 0.
\]
\end{lemma}
\begin{proof}
We decompose the expectation with respect to the value of the min, and use the
convexity of $f$:
\begin{eqnarray*}
f(X)-f(\min(c,Y))
&=& \oo{Y\le c}\paren*{f(X)-f(Y)} + \oo{Y>c}\paren*{f(X)-f(c)}\\
&\ge& \oo{Y\le c}f'(Y)\paren*{X-Y} + \oo{Y>c}(X-c)f'(c)\\
&\ge& \oo{Y\le c}f'(c)\paren*{X-Y} + \oo{Y>c}(X-c)f'(c)\\
&=& Xf'(c) - \min(Y,c)f'(c),
\end{eqnarray*}
where we used the fact that $f'$ is non-decreasing in the previous to last step.
Taking the expectation we get the result.
\end{proof}

\begin{lemma}\label{le:aux}
Let $P_g,P_d$ be two fixed distributions and $\beta\in(0,1)$. 
Assume
\[
P_d\left( \frac{dP_g}{dP_d}= 0\right) < \beta.
\]
Let $\M(P_d,\beta)$
be the set of all probability distributions $T$ such that $(1-\beta)dT \le dP_d$.
Then the following minimization problem:
\[
\min_{T\in \M(P_d,\beta)} D_f(T\,\|\,P_g)
\]
has the solution $T^*$ with density
\[
dT^*\bydef \min(dP_d/(1-\beta), \lambda^\dagger dP_g),
\]
where $\lambda^\dagger$ is the unique value in $[1,\infty)$ such that $\int dT^*=1$.
\end{lemma}
\begin{proof}
We will use Lemma \ref{le:imp2} with $X=dT(Z)/dP_g(Z)$, $Y=dP_d(Z)/\bigl((1-\beta)dP_g(Z)\bigr)$, and $c=\lambda^*$, $Z\sim P_g$.
We need to verify that assumptions of Lemma \ref{le:imp2} are satisfied.
Obviously, $Y\geq X$.
We need to show that there is a constant $c$ such that
\[
\int \min\left(c, \frac{dP_d}{(1-\beta) dP_g}\right)dP_g
=1.
\]
Rewriting this equation we get the following equivalent one:
\begin{equation}
\label{eq:middle-proof-normal}
\beta = 
\int \left( dP_d - \min\left(c(1-\beta)P_g, dP_d\right)\right)
=
(1-\beta)\int \left( \frac{1}{1-\beta} - c\frac{dP_g}{dP_d}\right)_+dP_d.
\end{equation}
Using the fact that
\[
P_d\left( \frac{dP_g}{dP_d}= 0\right) < \beta
\] we may apply Lemma \ref{le:lambda2} and conclude that there is a unique $c\in[1,\infty)$ satisfying \eqref{eq:middle-proof-normal}, which we denote $\lambda^\dagger$.
\end{proof}

To conclude the proof of Theorem \ref{th:opt2}, observe that
from Lemma \ref{le:aux}, by making the change of variable
$T=(P_d-\beta Q)/(1-\beta)$ we can rewrite the minimization problem as follows:
\[
\min_{Q:\; \beta dQ  \leq dP_d } D_{f^\circ}\paren*{P_g \,\|\, \frac{P_d-\beta Q}{1-\beta}}
\]
and we verify that the solution has the form
$dQ^{\dagger}_{\beta}=\frac{1}{\beta}\paren*{dP_d-\lambda^\dagger (1-\beta)dP_g}_+$.
Since this solution does not depend on $f$, the fact that we optimized $D_{f^\circ}$
is irrelevant and we get the same solution for $D_f$.

\section{$f$-Divergences}
\label{appendix:f-div}
\paragraph{Jensen-Shannon} 
This divergence corresponds to 
\[
D_f(P\|Q) = \mathrm{JS}(P,Q) = \int_{\X} f\left(\frac{dP}{dQ}(x)\right) dQ(x)
\]
with
\[
f(u) = -(u+1)\log \frac{u+1}{2} + u\log u.
\]
Indeed,
\begin{align*}
&JS(P,Q)
:=
\int_{\X} q(x)\left(
-\left(\frac{p(x)}{q(x)} + 1\right)\log\left(\frac{\frac{p(x)}{q(x)} + 1}{2}\right)
+
\frac{p(x)}{q(x)}\log\frac{p(x)}{q(z)}
\right)dx\\
&=
\int_{\X} q(x)\left(
\frac{p(x)}{q(x)}\log\frac{2q(x)}{p(x) + q(x)}
+
\log\frac{2q(x)}{p(x) + q(x)}
+
\frac{p(x)}{q(z)}\log\frac{p(x)}{q(z)}
\right)dx\\
&=\int_{\X} 
p(x)\log\frac{2q(x)}{p(x) + q(x)}
+
q(x)\log\frac{2q(x)}{p(x) + q(x)}
+
p(x)\log\frac{p(x)}{q(z)}
dx\\
&=
\mathrm{KL}\left(Q,\frac{P+Q}{2}\right)
+
\mathrm{KL}\left(P,\frac{P+Q}{2}\right).
\end{align*}

\section{Additional experimental results}
\label{sec:additional_experiments}

At each iteration of the boosting approach, different reweighting heuristics are possible. This section contains more complete results about the following three heuristics:
\begin{itemize}
  \item Constant $\beta$, and using the proposed reweighting scheme given $\beta$. See Table \ref{table:constant_beta}.
  \item Reweighting similar to ``Cascade GAN'' from \cite{wang2016ensembles}, i.e.
keep the top $x\%$ of examples, based on the discriminator corresponding to the \emph{previous} generator. See Table \ref{table:top_k_last}.
  \item Keep the top $x\%$ of examples, based on the discriminator corresponding to
the mixture of \emph{all previous} generators. See Table \ref{table:top_k_all}.
\end{itemize}

Note that when properly tuned, each reweighting scheme outperforms the baselines, and have similar performances when used with few iterations.
However, they require an additional parameter to tune, and are worse than the simple $\beta=1/t$ heuristic proposed above.

\begin{table}

\begin{center}
\ \hspace*{-1.5cm}
{\renewcommand{\arraystretch}{1.7}
\begin{tabular}{ | l | c|c|c|c|c|c | }
  \hline
     &   $Modes: 1$  & $Modes: 2$  & $Modes: 3$  & $Modes: 5$  & $Modes: 7$  & $Modes: 10$  \\ \hline
 Vanilla  &  \begin{minipage}{2cm}$0.97$ \footnotesize{$(0.9;1.0)$}\end{minipage}  &  \begin{minipage}{2cm}$0.88$ \footnotesize{$(0.4;1.0)$}\end{minipage}  &  \begin{minipage}{2cm}$0.63$ \footnotesize{$(0.5;1.0)$}\end{minipage}  &  \begin{minipage}{2cm}$0.72$ \footnotesize{$(0.5;0.8)$}\end{minipage}  &  \begin{minipage}{2cm}$0.58$ \footnotesize{$(0.4;0.8)$}\end{minipage}  &  \begin{minipage}{2cm}$0.59$ \footnotesize{$(0.2;0.7)$}\end{minipage}  \\ \hline
 Best of T (T=3)  &  \begin{minipage}{2cm}$0.99$ \footnotesize{$(1.0;1.0)$}\end{minipage}  &  \begin{minipage}{2cm}$0.96$ \footnotesize{$(0.9;1.0)$}\end{minipage}  &  \begin{minipage}{2cm}$0.91$ \footnotesize{$(0.7;1.0)$}\end{minipage}  &  \begin{minipage}{2cm}$0.80$ \footnotesize{$(0.7;0.9)$}\end{minipage}  &  \begin{minipage}{2cm}$0.84$ \footnotesize{$(0.7;0.9)$}\end{minipage}  &  \begin{minipage}{2cm}$0.70$ \footnotesize{$(0.6;0.8)$}\end{minipage}  \\ \hline
 Best of T (T=10)  &  \begin{minipage}{2cm}$0.99$ \footnotesize{$(1.0;1.0)$}\end{minipage}  &  \begin{minipage}{2cm}$0.99$ \footnotesize{$(1.0;1.0)$}\end{minipage}  &  \begin{minipage}{2cm}$0.98$ \footnotesize{$(0.8;1.0)$}\end{minipage}  &  \begin{minipage}{2cm}$0.80$ \footnotesize{$(0.8;0.9)$}\end{minipage}  &  \begin{minipage}{2cm}$0.87$ \footnotesize{$(0.8;0.9)$}\end{minipage}  &  \begin{minipage}{2cm}$0.71$ \footnotesize{$(0.7;0.8)$}\end{minipage}  \\ \hline
 Ensemble (T=3)  &  \begin{minipage}{2cm}$0.99$ \footnotesize{$(1.0;1.0)$}\end{minipage}  &  \begin{minipage}{2cm}$0.98$ \footnotesize{$(0.9;1.0)$}\end{minipage}  &  \begin{minipage}{2cm}$0.93$ \footnotesize{$(0.8;1.0)$}\end{minipage}  &  \begin{minipage}{2cm}$0.78$ \footnotesize{$(0.6;1.0)$}\end{minipage}  &  \begin{minipage}{2cm}$0.85$ \footnotesize{$(0.6;1.0)$}\end{minipage}  &  \begin{minipage}{2cm}$0.80$ \footnotesize{$(0.6;1.0)$}\end{minipage}  \\ \hline
 Ensemble (T=10)  &  \begin{minipage}{2cm}$1.00$ \footnotesize{$(1.0;1.0)$}\end{minipage}  &  \begin{minipage}{2cm}$0.99$ \footnotesize{$(1.0;1.0)$}\end{minipage}  &  \begin{minipage}{2cm}$1.00$ \footnotesize{$(1.0;1.0)$}\end{minipage}  &  \begin{minipage}{2cm}$0.91$ \footnotesize{$(0.8;1.0)$}\end{minipage}  &  \begin{minipage}{2cm}$0.88$ \footnotesize{$(0.8;1.0)$}\end{minipage}  &  \begin{minipage}{2cm}$0.89$ \footnotesize{$(0.7;1.0)$}\end{minipage}  \\ \hline
 Boosted (T=3)  &  \begin{minipage}{2cm}$0.99$ \footnotesize{$(1.0;1.0)$}\end{minipage}  &  \begin{minipage}{2cm}$0.99$ \footnotesize{$(0.9;1.0)$}\end{minipage}  &  \begin{minipage}{2cm}$0.98$ \footnotesize{$(0.9;1.0)$}\end{minipage}  &  \begin{minipage}{2cm}$0.91$ \footnotesize{$(0.8;1.0)$}\end{minipage}  &  \begin{minipage}{2cm}$0.91$ \footnotesize{$(0.8;1.0)$}\end{minipage}  &  \begin{minipage}{2cm}$0.86$ \footnotesize{$(0.7;1.0)$}\end{minipage}  \\ \hline
 Boosted (T=10)  &  \begin{minipage}{2cm}$1.00$ \footnotesize{$(1.0;1.0)$}\end{minipage}  &  \begin{minipage}{2cm}$1.00$ \footnotesize{$(1.0;1.0)$}\end{minipage}  &  \begin{minipage}{2cm}$1.00$ \footnotesize{$(1.0;1.0)$}\end{minipage}  &  \begin{minipage}{2cm}$1.00$ \footnotesize{$(1.0;1.0)$}\end{minipage}  &  \begin{minipage}{2cm}$1.00$ \footnotesize{$(1.0;1.0)$}\end{minipage}  &  \begin{minipage}{2cm}$1.00$ \footnotesize{$(1.0;1.0)$}\end{minipage}  \\ \hline
  \end{tabular}}
\end{center}

\begin{center}
\ \hspace*{-1.5cm}
{\renewcommand{\arraystretch}{1.7}
\begin{tabular}{ | l | c|c|c|c|c|c | }
  \hline
     &   $Modes: 1$  & $Modes: 2$  & $Modes: 3$  & $Modes: 5$  & $Modes: 7$  & $Modes: 10$  \\ \hline
 Vanilla  &  \begin{minipage}{2cm}$-4.49$ \footnotesize{$(-5.4;-4.4)$}\end{minipage}  &  \begin{minipage}{2cm}$-6.02$ \footnotesize{$(-86.8;-5.3)$}\end{minipage}  &  \begin{minipage}{2cm}$-16.03$ \footnotesize{$(-59.6;-5.5)$}\end{minipage}  &  \begin{minipage}{2cm}$-23.65$ \footnotesize{$(-118.8;-5.7)$}\end{minipage}  &  \begin{minipage}{2cm}$-126.87$ \footnotesize{$(-250.4;-12.8)$}\end{minipage}  &  \begin{minipage}{2cm}$-55.51$ \footnotesize{$(-185.2;-11.2)$}\end{minipage}  \\ \hline
 Best of T (T=3)  &  \begin{minipage}{2cm}$-4.39$ \footnotesize{$(-4.6;-4.3)$}\end{minipage}  &  \begin{minipage}{2cm}$-5.40$ \footnotesize{$(-24.3;-5.2)$}\end{minipage}  &  \begin{minipage}{2cm}$-5.57$ \footnotesize{$(-23.5;-5.4)$}\end{minipage}  &  \begin{minipage}{2cm}$-9.91$ \footnotesize{$(-35.8;-5.1)$}\end{minipage}  &  \begin{minipage}{2cm}$-36.94$ \footnotesize{$(-90.0;-9.7)$}\end{minipage}  &  \begin{minipage}{2cm}$-19.12$ \footnotesize{$(-59.2;-9.7)$}\end{minipage}  \\ \hline
 Best of T (T=10)  &  \begin{minipage}{2cm}$-4.34$ \footnotesize{$(-4.4;-4.3)$}\end{minipage}  &  \begin{minipage}{2cm}$-5.24$ \footnotesize{$(-5.4;-5.2)$}\end{minipage}  &  \begin{minipage}{2cm}$-5.45$ \footnotesize{$(-5.6;-5.3)$}\end{minipage}  &  \begin{minipage}{2cm}$-5.49$ \footnotesize{$(-9.4;-5.0)$}\end{minipage}  &  \begin{minipage}{2cm}$-9.72$ \footnotesize{$(-17.3;-6.5)$}\end{minipage}  &  \begin{minipage}{2cm}$-9.12$ \footnotesize{$(-16.8;-6.6)$}\end{minipage}  \\ \hline
 Ensemble (T=3)  &  \begin{minipage}{2cm}$-4.46$ \footnotesize{$(-4.8;-4.4)$}\end{minipage}  &  \begin{minipage}{2cm}$-5.59$ \footnotesize{$(-6.6;-5.2)$}\end{minipage}  &  \begin{minipage}{2cm}$-4.78$ \footnotesize{$(-5.5;-4.6)$}\end{minipage}  &  \begin{minipage}{2cm}$-14.71$ \footnotesize{$(-51.9;-5.4)$}\end{minipage}  &  \begin{minipage}{2cm}$-6.70$ \footnotesize{$(-28.7;-5.5)$}\end{minipage}  &  \begin{minipage}{2cm}$-8.59$ \footnotesize{$(-25.4;-6.1)$}\end{minipage}  \\ \hline
 Ensemble (T=10)  &  \begin{minipage}{2cm}$-4.52$ \footnotesize{$(-4.7;-4.4)$}\end{minipage}  &  \begin{minipage}{2cm}$-5.49$ \footnotesize{$(-6.6;-5.2)$}\end{minipage}  &  \begin{minipage}{2cm}$-4.98$ \footnotesize{$(-6.5;-4.6)$}\end{minipage}  &  \begin{minipage}{2cm}$-5.44$ \footnotesize{$(-6.0;-5.2)$}\end{minipage}  &  \begin{minipage}{2cm}$-5.82$ \footnotesize{$(-6.4;-5.5)$}\end{minipage}  &  \begin{minipage}{2cm}$-6.08$ \footnotesize{$(-6.3;-5.7)$}\end{minipage}  \\ \hline
 Boosted (T=3)  &  \begin{minipage}{2cm}$-4.50$ \footnotesize{$(-4.8;-4.4)$}\end{minipage}  &  \begin{minipage}{2cm}$-5.32$ \footnotesize{$(-5.8;-5.2)$}\end{minipage}  &  \begin{minipage}{2cm}$-4.80$ \footnotesize{$(-5.8;-4.6)$}\end{minipage}  &  \begin{minipage}{2cm}$-5.39$ \footnotesize{$(-19.3;-5.1)$}\end{minipage}  &  \begin{minipage}{2cm}$-5.56$ \footnotesize{$(-12.4;-5.2)$}\end{minipage}  &  \begin{minipage}{2cm}$-8.03$ \footnotesize{$(-28.7;-6.1)$}\end{minipage}  \\ \hline
 Boosted (T=10)  &  \begin{minipage}{2cm}$-4.55$ \footnotesize{$(-4.6;-4.4)$}\end{minipage}  &  \begin{minipage}{2cm}$-5.30$ \footnotesize{$(-5.5;-5.2)$}\end{minipage}  &  \begin{minipage}{2cm}$-5.07$ \footnotesize{$(-5.6;-4.7)$}\end{minipage}  &  \begin{minipage}{2cm}$-5.25$ \footnotesize{$(-5.5;-4.6)$}\end{minipage}  &  \begin{minipage}{2cm}$-5.03$ \footnotesize{$(-5.5;-4.8)$}\end{minipage}  &  \begin{minipage}{2cm}$-5.92$ \footnotesize{$(-6.2;-5.6)$}\end{minipage}  \\ \hline
  \end{tabular}}
\end{center}

 \caption{Performance of the different algorithms on varying number of mixtures of Gaussians.
   The reported scores are the median and interval defined by the 5\% and 95\% percentile (in parenthesis) (see Section \ref{sec:metrics}), over 35 runs for each setting.
   The top table reports the coverage $C$, probability mass of $P_d$ covered
   by the $5th$ percentile of $P_g$ defined in Section \ref{experiments}. The bottom table reports the 
   log likelihood of the true data under the model~$P_g$.
   Note that the $95\%$ interval is not the usual confidence interval measuring the
   variance of the experiment itself, but rather measures the stability of the
   different algorithms (would remain even if each experiment was run an infinite
   number of times).
   Both the ensemble and the boosting approaches significantly outperform the vanilla GAN even with just three iterations (i.e. just
   two additional components). The boosting approach converges faster to the optimal coverage.
 }

\label{table:simple_comparison_complete}

\end{table}

\begin{table}

\begin{center}
\ \hspace*{-1.5cm}
{\renewcommand{\arraystretch}{1.7}
\begin{tabular}{ | l | c|c|c|c|c|c | }
  \hline
     &   $Modes: 1$  & $Modes: 2$  & $Modes: 3$  & $Modes: 5$  & $Modes: 7$  & $Modes: 10$  \\ \hline
 Vanilla  &  \begin{minipage}{2cm}$0.98$ \footnotesize{$(0.9;1.0)$}\end{minipage}  &  \begin{minipage}{2cm}$0.86$ \footnotesize{$(0.5;1.0)$}\end{minipage}  &  \begin{minipage}{2cm}$0.66$ \footnotesize{$(0.5;1.0)$}\end{minipage}  &  \begin{minipage}{2cm}$0.61$ \footnotesize{$(0.5;0.8)$}\end{minipage}  &  \begin{minipage}{2cm}$0.55$ \footnotesize{$(0.4;0.7)$}\end{minipage}  &  \begin{minipage}{2cm}$0.58$ \footnotesize{$(0.3;0.8)$}\end{minipage}  \\ \hline
 Boosted (T=3)  &  \begin{minipage}{2cm}$0.99$ \footnotesize{$(1.0;1.0)$}\end{minipage}  &  \begin{minipage}{2cm}$0.98$ \footnotesize{$(0.9;1.0)$}\end{minipage}  &  \begin{minipage}{2cm}$0.98$ \footnotesize{$(0.9;1.0)$}\end{minipage}  &  \begin{minipage}{2cm}$0.93$ \footnotesize{$(0.8;1.0)$}\end{minipage}  &  \begin{minipage}{2cm}$0.97$ \footnotesize{$(0.8;1.0)$}\end{minipage}  &  \begin{minipage}{2cm}$0.87$ \footnotesize{$(0.6;1.0)$}\end{minipage}  \\ \hline
 Boosted (T=10)  &  \begin{minipage}{2cm}$1.00$ \footnotesize{$(1.0;1.0)$}\end{minipage}  &  \begin{minipage}{2cm}$0.99$ \footnotesize{$(1.0;1.0)$}\end{minipage}  &  \begin{minipage}{2cm}$1.00$ \footnotesize{$(1.0;1.0)$}\end{minipage}  &  \begin{minipage}{2cm}$0.99$ \footnotesize{$(0.9;1.0)$}\end{minipage}  &  \begin{minipage}{2cm}$0.99$ \footnotesize{$(0.8;1.0)$}\end{minipage}  &  \begin{minipage}{2cm}$0.97$ \footnotesize{$(0.8;1.0)$}\end{minipage}  \\ \hline
 Beta0.2 (T=3)  &  \begin{minipage}{2cm}$0.99$ \footnotesize{$(1.0;1.0)$}\end{minipage}  &  \begin{minipage}{2cm}$0.97$ \footnotesize{$(0.9;1.0)$}\end{minipage}  &  \begin{minipage}{2cm}$0.97$ \footnotesize{$(0.9;1.0)$}\end{minipage}  &  \begin{minipage}{2cm}$0.95$ \footnotesize{$(0.8;1.0)$}\end{minipage}  &  \begin{minipage}{2cm}$0.96$ \footnotesize{$(0.7;1.0)$}\end{minipage}  &  \begin{minipage}{2cm}$0.88$ \footnotesize{$(0.7;1.0)$}\end{minipage}  \\ \hline
 Beta0.2 (T=10)  &  \begin{minipage}{2cm}$0.99$ \footnotesize{$(1.0;1.0)$}\end{minipage}  &  \begin{minipage}{2cm}$0.99$ \footnotesize{$(1.0;1.0)$}\end{minipage}  &  \begin{minipage}{2cm}$0.99$ \footnotesize{$(1.0;1.0)$}\end{minipage}  &  \begin{minipage}{2cm}$1.00$ \footnotesize{$(1.0;1.0)$}\end{minipage}  &  \begin{minipage}{2cm}$1.00$ \footnotesize{$(0.9;1.0)$}\end{minipage}  &  \begin{minipage}{2cm}$1.00$ \footnotesize{$(0.9;1.0)$}\end{minipage}  \\ \hline
 Beta0.3 (T=3)  &  \begin{minipage}{2cm}$0.99$ \footnotesize{$(1.0;1.0)$}\end{minipage}  &  \begin{minipage}{2cm}$0.98$ \footnotesize{$(0.9;1.0)$}\end{minipage}  &  \begin{minipage}{2cm}$0.98$ \footnotesize{$(0.9;1.0)$}\end{minipage}  &  \begin{minipage}{2cm}$0.96$ \footnotesize{$(0.8;1.0)$}\end{minipage}  &  \begin{minipage}{2cm}$0.96$ \footnotesize{$(0.6;1.0)$}\end{minipage}  &  \begin{minipage}{2cm}$0.88$ \footnotesize{$(0.7;1.0)$}\end{minipage}  \\ \hline
 Beta0.3 (T=10)  &  \begin{minipage}{2cm}$1.00$ \footnotesize{$(1.0;1.0)$}\end{minipage}  &  \begin{minipage}{2cm}$0.99$ \footnotesize{$(1.0;1.0)$}\end{minipage}  &  \begin{minipage}{2cm}$0.99$ \footnotesize{$(1.0;1.0)$}\end{minipage}  &  \begin{minipage}{2cm}$1.00$ \footnotesize{$(1.0;1.0)$}\end{minipage}  &  \begin{minipage}{2cm}$1.00$ \footnotesize{$(0.9;1.0)$}\end{minipage}  &  \begin{minipage}{2cm}$0.99$ \footnotesize{$(0.9;1.0)$}\end{minipage}  \\ \hline
 Beta0.4 (T=3)  &  \begin{minipage}{2cm}$0.99$ \footnotesize{$(1.0;1.0)$}\end{minipage}  &  \begin{minipage}{2cm}$0.98$ \footnotesize{$(0.9;1.0)$}\end{minipage}  &  \begin{minipage}{2cm}$0.95$ \footnotesize{$(0.9;1.0)$}\end{minipage}  &  \begin{minipage}{2cm}$0.94$ \footnotesize{$(0.8;1.0)$}\end{minipage}  &  \begin{minipage}{2cm}$0.89$ \footnotesize{$(0.7;1.0)$}\end{minipage}  &  \begin{minipage}{2cm}$0.89$ \footnotesize{$(0.7;1.0)$}\end{minipage}  \\ \hline
 Beta0.4 (T=10)  &  \begin{minipage}{2cm}$0.99$ \footnotesize{$(1.0;1.0)$}\end{minipage}  &  \begin{minipage}{2cm}$0.99$ \footnotesize{$(0.9;1.0)$}\end{minipage}  &  \begin{minipage}{2cm}$0.96$ \footnotesize{$(0.9;1.0)$}\end{minipage}  &  \begin{minipage}{2cm}$0.97$ \footnotesize{$(0.8;1.0)$}\end{minipage}  &  \begin{minipage}{2cm}$0.99$ \footnotesize{$(0.8;1.0)$}\end{minipage}  &  \begin{minipage}{2cm}$0.90$ \footnotesize{$(0.8;1.0)$}\end{minipage}  \\ \hline
 Beta0.5 (T=3)  &  \begin{minipage}{2cm}$0.99$ \footnotesize{$(1.0;1.0)$}\end{minipage}  &  \begin{minipage}{2cm}$0.98$ \footnotesize{$(0.9;1.0)$}\end{minipage}  &  \begin{minipage}{2cm}$0.97$ \footnotesize{$(0.8;1.0)$}\end{minipage}  &  \begin{minipage}{2cm}$0.82$ \footnotesize{$(0.8;1.0)$}\end{minipage}  &  \begin{minipage}{2cm}$0.86$ \footnotesize{$(0.7;1.0)$}\end{minipage}  &  \begin{minipage}{2cm}$0.81$ \footnotesize{$(0.6;1.0)$}\end{minipage}  \\ \hline
 Beta0.5 (T=10)  &  \begin{minipage}{2cm}$0.99$ \footnotesize{$(1.0;1.0)$}\end{minipage}  &  \begin{minipage}{2cm}$0.98$ \footnotesize{$(0.9;1.0)$}\end{minipage}  &  \begin{minipage}{2cm}$0.97$ \footnotesize{$(0.9;1.0)$}\end{minipage}  &  \begin{minipage}{2cm}$0.84$ \footnotesize{$(0.8;1.0)$}\end{minipage}  &  \begin{minipage}{2cm}$0.87$ \footnotesize{$(0.7;1.0)$}\end{minipage}  &  \begin{minipage}{2cm}$0.91$ \footnotesize{$(0.8;1.0)$}\end{minipage}  \\ \hline
  \end{tabular}}
\end{center}

\begin{center}
\ \hspace*{-1.5cm}
{\renewcommand{\arraystretch}{1.7}
\begin{tabular}{ | l | c|c|c|c|c|c | }
  \hline
     &   $Modes: 1$  & $Modes: 2$  & $Modes: 3$  & $Modes: 5$  & $Modes: 7$  & $Modes: 10$  \\ \hline
 Vanilla  &  \begin{minipage}{2cm}$-4.50$ \footnotesize{$(-5.0;-4.4)$}\end{minipage}  &  \begin{minipage}{2cm}$-5.65$ \footnotesize{$(-72.7;-5.1)$}\end{minipage}  &  \begin{minipage}{2cm}$-19.63$ \footnotesize{$(-62.1;-5.6)$}\end{minipage}  &  \begin{minipage}{2cm}$-28.16$ \footnotesize{$(-293.1;-16.3)$}\end{minipage}  &  \begin{minipage}{2cm}$-56.94$ \footnotesize{$(-248.1;-14.3)$}\end{minipage}  &  \begin{minipage}{2cm}$-71.11$ \footnotesize{$(-184.8;-12.5)$}\end{minipage}  \\ \hline
 Boosted (T=3)  &  \begin{minipage}{2cm}$-4.56$ \footnotesize{$(-4.9;-4.4)$}\end{minipage}  &  \begin{minipage}{2cm}$-5.55$ \footnotesize{$(-5.9;-5.2)$}\end{minipage}  &  \begin{minipage}{2cm}$-5.01$ \footnotesize{$(-6.7;-4.7)$}\end{minipage}  &  \begin{minipage}{2cm}$-5.49$ \footnotesize{$(-18.7;-4.9)$}\end{minipage}  &  \begin{minipage}{2cm}$-5.60$ \footnotesize{$(-14.5;-5.0)$}\end{minipage}  &  \begin{minipage}{2cm}$-6.86$ \footnotesize{$(-47.3;-5.6)$}\end{minipage}  \\ \hline
 Boosted (T=10)  &  \begin{minipage}{2cm}$-4.56$ \footnotesize{$(-4.7;-4.5)$}\end{minipage}  &  \begin{minipage}{2cm}$-5.46$ \footnotesize{$(-5.6;-5.3)$}\end{minipage}  &  \begin{minipage}{2cm}$-5.08$ \footnotesize{$(-5.8;-4.7)$}\end{minipage}  &  \begin{minipage}{2cm}$-5.04$ \footnotesize{$(-5.5;-4.6)$}\end{minipage}  &  \begin{minipage}{2cm}$-5.51$ \footnotesize{$(-5.9;-5.1)$}\end{minipage}  &  \begin{minipage}{2cm}$-5.51$ \footnotesize{$(-6.0;-5.2)$}\end{minipage}  \\ \hline
 Beta0.2 (T=3)  &  \begin{minipage}{2cm}$-4.52$ \footnotesize{$(-4.8;-4.4)$}\end{minipage}  &  \begin{minipage}{2cm}$-5.31$ \footnotesize{$(-5.6;-5.1)$}\end{minipage}  &  \begin{minipage}{2cm}$-4.85$ \footnotesize{$(-6.3;-4.6)$}\end{minipage}  &  \begin{minipage}{2cm}$-5.33$ \footnotesize{$(-14.4;-4.8)$}\end{minipage}  &  \begin{minipage}{2cm}$-5.68$ \footnotesize{$(-26.2;-5.2)$}\end{minipage}  &  \begin{minipage}{2cm}$-6.13$ \footnotesize{$(-32.7;-5.7)$}\end{minipage}  \\ \hline
 Beta0.2 (T=10)  &  \begin{minipage}{2cm}$-4.58$ \footnotesize{$(-4.8;-4.5)$}\end{minipage}  &  \begin{minipage}{2cm}$-5.30$ \footnotesize{$(-5.5;-5.2)$}\end{minipage}  &  \begin{minipage}{2cm}$-4.94$ \footnotesize{$(-6.6;-4.6)$}\end{minipage}  &  \begin{minipage}{2cm}$-5.23$ \footnotesize{$(-5.5;-4.7)$}\end{minipage}  &  \begin{minipage}{2cm}$-5.60$ \footnotesize{$(-6.0;-5.3)$}\end{minipage}  &  \begin{minipage}{2cm}$-5.98$ \footnotesize{$(-6.1;-5.7)$}\end{minipage}  \\ \hline
 Beta0.3 (T=3)  &  \begin{minipage}{2cm}$-4.60$ \footnotesize{$(-4.9;-4.4)$}\end{minipage}  &  \begin{minipage}{2cm}$-5.34$ \footnotesize{$(-5.7;-5.2)$}\end{minipage}  &  \begin{minipage}{2cm}$-5.41$ \footnotesize{$(-5.7;-5.1)$}\end{minipage}  &  \begin{minipage}{2cm}$-5.33$ \footnotesize{$(-12.9;-4.9)$}\end{minipage}  &  \begin{minipage}{2cm}$-5.68$ \footnotesize{$(-11.0;-5.4)$}\end{minipage}  &  \begin{minipage}{2cm}$-6.41$ \footnotesize{$(-29.2;-5.6)$}\end{minipage}  \\ \hline
 Beta0.3 (T=10)  &  \begin{minipage}{2cm}$-4.57$ \footnotesize{$(-4.8;-4.4)$}\end{minipage}  &  \begin{minipage}{2cm}$-5.37$ \footnotesize{$(-5.5;-5.2)$}\end{minipage}  &  \begin{minipage}{2cm}$-5.27$ \footnotesize{$(-5.6;-5.0)$}\end{minipage}  &  \begin{minipage}{2cm}$-5.26$ \footnotesize{$(-5.6;-5.0)$}\end{minipage}  &  \begin{minipage}{2cm}$-5.71$ \footnotesize{$(-6.0;-5.3)$}\end{minipage}  &  \begin{minipage}{2cm}$-5.82$ \footnotesize{$(-6.1;-5.4)$}\end{minipage}  \\ \hline
 Beta0.4 (T=3)  &  \begin{minipage}{2cm}$-4.62$ \footnotesize{$(-4.9;-4.4)$}\end{minipage}  &  \begin{minipage}{2cm}$-5.36$ \footnotesize{$(-5.6;-5.1)$}\end{minipage}  &  \begin{minipage}{2cm}$-4.74$ \footnotesize{$(-5.3;-4.6)$}\end{minipage}  &  \begin{minipage}{2cm}$-5.34$ \footnotesize{$(-26.2;-4.9)$}\end{minipage}  &  \begin{minipage}{2cm}$-5.77$ \footnotesize{$(-37.3;-5.1)$}\end{minipage}  &  \begin{minipage}{2cm}$-12.37$ \footnotesize{$(-75.9;-5.9)$}\end{minipage}  \\ \hline
 Beta0.4 (T=10)  &  \begin{minipage}{2cm}$-4.49$ \footnotesize{$(-4.7;-4.4)$}\end{minipage}  &  \begin{minipage}{2cm}$-5.40$ \footnotesize{$(-5.7;-5.3)$}\end{minipage}  &  \begin{minipage}{2cm}$-5.08$ \footnotesize{$(-6.9;-4.7)$}\end{minipage}  &  \begin{minipage}{2cm}$-5.49$ \footnotesize{$(-5.9;-5.2)$}\end{minipage}  &  \begin{minipage}{2cm}$-5.43$ \footnotesize{$(-6.0;-5.1)$}\end{minipage}  &  \begin{minipage}{2cm}$-5.68$ \footnotesize{$(-6.2;-5.2)$}\end{minipage}  \\ \hline
 Beta0.5 (T=3)  &  \begin{minipage}{2cm}$-4.60$ \footnotesize{$(-4.9;-4.4)$}\end{minipage}  &  \begin{minipage}{2cm}$-5.40$ \footnotesize{$(-5.7;-5.3)$}\end{minipage}  &  \begin{minipage}{2cm}$-4.77$ \footnotesize{$(-5.4;-4.6)$}\end{minipage}  &  \begin{minipage}{2cm}$-5.63$ \footnotesize{$(-24.5;-5.2)$}\end{minipage}  &  \begin{minipage}{2cm}$-6.05$ \footnotesize{$(-17.9;-5.5)$}\end{minipage}  &  \begin{minipage}{2cm}$-8.29$ \footnotesize{$(-23.1;-6.1)$}\end{minipage}  \\ \hline
 Beta0.5 (T=10)  &  \begin{minipage}{2cm}$-4.62$ \footnotesize{$(-4.8;-4.4)$}\end{minipage}  &  \begin{minipage}{2cm}$-5.43$ \footnotesize{$(-5.7;-5.2)$}\end{minipage}  &  \begin{minipage}{2cm}$-5.12$ \footnotesize{$(-6.6;-4.7)$}\end{minipage}  &  \begin{minipage}{2cm}$-5.48$ \footnotesize{$(-8.4;-5.1)$}\end{minipage}  &  \begin{minipage}{2cm}$-5.85$ \footnotesize{$(-6.1;-5.3)$}\end{minipage}  &  \begin{minipage}{2cm}$-6.31$ \footnotesize{$(-7.7;-6.0)$}\end{minipage}  \\ \hline
  \end{tabular}}
\end{center}

\caption{Performance with constant $\beta$, exploring a range of possible values.
   The reported scores are the median and interval defined by the 5\% and 95\% percentile (in parenthesis) (see Section  \ref{sec:metrics}), over 35 runs for each setting.
   The top table reports the coverage $C$, probability mass of $P_d$ covered
   by the $5th$ percentile of $P_g$ defined in Section \ref{experiments}. The bottom table reports the 
   log likelihood of the true data under $P_g$.
   }

\label{table:constant_beta}
\end{table}

\begin{table}

\begin{center}
\ \hspace*{-1.5cm}
{\renewcommand{\arraystretch}{1.7}
\begin{tabular}{ | l | c|c|c|c|c|c | }
  \hline
     &   $Modes: 1$  & $Modes: 2$  & $Modes: 3$  & $Modes: 5$  & $Modes: 7$  & $Modes: 10$  \\ \hline
 Vanilla  &  \begin{minipage}{2cm}$0.96$ \footnotesize{$(0.9;1.0)$}\end{minipage}  &  \begin{minipage}{2cm}$0.90$ \footnotesize{$(0.5;1.0)$}\end{minipage}  &  \begin{minipage}{2cm}$0.65$ \footnotesize{$(0.5;1.0)$}\end{minipage}  &  \begin{minipage}{2cm}$0.61$ \footnotesize{$(0.5;0.8)$}\end{minipage}  &  \begin{minipage}{2cm}$0.69$ \footnotesize{$(0.3;0.8)$}\end{minipage}  &  \begin{minipage}{2cm}$0.59$ \footnotesize{$(0.3;0.7)$}\end{minipage}  \\ \hline
 Boosted (T=3)  &  \begin{minipage}{2cm}$0.99$ \footnotesize{$(1.0;1.0)$}\end{minipage}  &  \begin{minipage}{2cm}$0.98$ \footnotesize{$(0.9;1.0)$}\end{minipage}  &  \begin{minipage}{2cm}$0.98$ \footnotesize{$(0.9;1.0)$}\end{minipage}  &  \begin{minipage}{2cm}$0.93$ \footnotesize{$(0.8;1.0)$}\end{minipage}  &  \begin{minipage}{2cm}$0.97$ \footnotesize{$(0.8;1.0)$}\end{minipage}  &  \begin{minipage}{2cm}$0.87$ \footnotesize{$(0.6;1.0)$}\end{minipage}  \\ \hline
 Boosted (T=10)  &  \begin{minipage}{2cm}$1.00$ \footnotesize{$(1.0;1.0)$}\end{minipage}  &  \begin{minipage}{2cm}$0.99$ \footnotesize{$(1.0;1.0)$}\end{minipage}  &  \begin{minipage}{2cm}$1.00$ \footnotesize{$(1.0;1.0)$}\end{minipage}  &  \begin{minipage}{2cm}$0.99$ \footnotesize{$(0.9;1.0)$}\end{minipage}  &  \begin{minipage}{2cm}$0.99$ \footnotesize{$(0.8;1.0)$}\end{minipage}  &  \begin{minipage}{2cm}$0.97$ \footnotesize{$(0.8;1.0)$}\end{minipage}  \\ \hline
 TopKLast0.1 (T=3)  &  \begin{minipage}{2cm}$0.98$ \footnotesize{$(0.9;1.0)$}\end{minipage}  &  \begin{minipage}{2cm}$0.93$ \footnotesize{$(0.8;1.0)$}\end{minipage}  &  \begin{minipage}{2cm}$0.89$ \footnotesize{$(0.6;1.0)$}\end{minipage}  &  \begin{minipage}{2cm}$0.72$ \footnotesize{$(0.5;1.0)$}\end{minipage}  &  \begin{minipage}{2cm}$0.68$ \footnotesize{$(0.5;0.9)$}\end{minipage}  &  \begin{minipage}{2cm}$0.51$ \footnotesize{$(0.4;0.7)$}\end{minipage}  \\ \hline
 TopKLast0.1 (T=10)  &  \begin{minipage}{2cm}$0.99$ \footnotesize{$(0.9;1.0)$}\end{minipage}  &  \begin{minipage}{2cm}$0.97$ \footnotesize{$(0.8;1.0)$}\end{minipage}  &  \begin{minipage}{2cm}$0.90$ \footnotesize{$(0.7;1.0)$}\end{minipage}  &  \begin{minipage}{2cm}$0.67$ \footnotesize{$(0.4;0.9)$}\end{minipage}  &  \begin{minipage}{2cm}$0.61$ \footnotesize{$(0.5;0.8)$}\end{minipage}  &  \begin{minipage}{2cm}$0.58$ \footnotesize{$(0.4;0.8)$}\end{minipage}  \\ \hline
 TopKLast0.3 (T=3)  &  \begin{minipage}{2cm}$0.99$ \footnotesize{$(0.9;1.0)$}\end{minipage}  &  \begin{minipage}{2cm}$0.97$ \footnotesize{$(0.9;1.0)$}\end{minipage}  &  \begin{minipage}{2cm}$0.93$ \footnotesize{$(0.7;1.0)$}\end{minipage}  &  \begin{minipage}{2cm}$0.81$ \footnotesize{$(0.7;1.0)$}\end{minipage}  &  \begin{minipage}{2cm}$0.84$ \footnotesize{$(0.7;1.0)$}\end{minipage}  &  \begin{minipage}{2cm}$0.78$ \footnotesize{$(0.5;1.0)$}\end{minipage}  \\ \hline
 TopKLast0.3 (T=10)  &  \begin{minipage}{2cm}$0.99$ \footnotesize{$(1.0;1.0)$}\end{minipage}  &  \begin{minipage}{2cm}$0.98$ \footnotesize{$(0.9;1.0)$}\end{minipage}  &  \begin{minipage}{2cm}$0.95$ \footnotesize{$(0.7;1.0)$}\end{minipage}  &  \begin{minipage}{2cm}$0.94$ \footnotesize{$(0.7;1.0)$}\end{minipage}  &  \begin{minipage}{2cm}$0.89$ \footnotesize{$(0.7;1.0)$}\end{minipage}  &  \begin{minipage}{2cm}$0.88$ \footnotesize{$(0.7;1.0)$}\end{minipage}  \\ \hline
 TopKLast0.5 (T=3)  &  \begin{minipage}{2cm}$0.98$ \footnotesize{$(0.9;1.0)$}\end{minipage}  &  \begin{minipage}{2cm}$0.98$ \footnotesize{$(0.9;1.0)$}\end{minipage}  &  \begin{minipage}{2cm}$0.95$ \footnotesize{$(0.9;1.0)$}\end{minipage}  &  \begin{minipage}{2cm}$0.95$ \footnotesize{$(0.8;1.0)$}\end{minipage}  &  \begin{minipage}{2cm}$0.86$ \footnotesize{$(0.7;1.0)$}\end{minipage}  &  \begin{minipage}{2cm}$0.86$ \footnotesize{$(0.6;0.9)$}\end{minipage}  \\ \hline
 TopKLast0.5 (T=10)  &  \begin{minipage}{2cm}$0.99$ \footnotesize{$(1.0;1.0)$}\end{minipage}  &  \begin{minipage}{2cm}$0.98$ \footnotesize{$(0.9;1.0)$}\end{minipage}  &  \begin{minipage}{2cm}$0.98$ \footnotesize{$(1.0;1.0)$}\end{minipage}  &  \begin{minipage}{2cm}$0.99$ \footnotesize{$(0.8;1.0)$}\end{minipage}  &  \begin{minipage}{2cm}$0.99$ \footnotesize{$(0.8;1.0)$}\end{minipage}  &  \begin{minipage}{2cm}$1.00$ \footnotesize{$(0.8;1.0)$}\end{minipage}  \\ \hline
 TopKLast0.7 (T=3)  &  \begin{minipage}{2cm}$0.98$ \footnotesize{$(1.0;1.0)$}\end{minipage}  &  \begin{minipage}{2cm}$0.98$ \footnotesize{$(0.9;1.0)$}\end{minipage}  &  \begin{minipage}{2cm}$0.94$ \footnotesize{$(0.9;1.0)$}\end{minipage}  &  \begin{minipage}{2cm}$0.83$ \footnotesize{$(0.7;1.0)$}\end{minipage}  &  \begin{minipage}{2cm}$0.87$ \footnotesize{$(0.6;1.0)$}\end{minipage}  &  \begin{minipage}{2cm}$0.82$ \footnotesize{$(0.7;1.0)$}\end{minipage}  \\ \hline
 TopKLast0.7 (T=10)  &  \begin{minipage}{2cm}$0.99$ \footnotesize{$(1.0;1.0)$}\end{minipage}  &  \begin{minipage}{2cm}$0.99$ \footnotesize{$(1.0;1.0)$}\end{minipage}  &  \begin{minipage}{2cm}$1.00$ \footnotesize{$(1.0;1.0)$}\end{minipage}  &  \begin{minipage}{2cm}$0.98$ \footnotesize{$(0.8;1.0)$}\end{minipage}  &  \begin{minipage}{2cm}$0.99$ \footnotesize{$(0.9;1.0)$}\end{minipage}  &  \begin{minipage}{2cm}$0.95$ \footnotesize{$(0.8;1.0)$}\end{minipage}  \\ \hline
  \end{tabular}}
\end{center}

\begin{center}
\ \hspace*{-1.5cm}
{\renewcommand{\arraystretch}{1.7}
\begin{tabular}{ | l | c|c|c|c|c|c | }
  \hline
     &   $Modes: 1$  & $Modes: 2$  & $Modes: 3$  & $Modes: 5$  & $Modes: 7$  & $Modes: 10$  \\ \hline
 Vanilla  &  \begin{minipage}{2cm}$-4.94$ \footnotesize{$(-5.5;-4.4)$}\end{minipage}  &  \begin{minipage}{2cm}$-6.18$ \footnotesize{$(-51.7;-5.6)$}\end{minipage}  &  \begin{minipage}{2cm}$-31.85$ \footnotesize{$(-100.3;-5.8)$}\end{minipage}  &  \begin{minipage}{2cm}$-47.73$ \footnotesize{$(-155.1;-14.2)$}\end{minipage}  &  \begin{minipage}{2cm}$-107.36$ \footnotesize{$(-390.8;-14.8)$}\end{minipage}  &  \begin{minipage}{2cm}$-59.19$ \footnotesize{$(-264.3;-18.8)$}\end{minipage}  \\ \hline
 Boosted (T=3)  &  \begin{minipage}{2cm}$-4.56$ \footnotesize{$(-4.9;-4.4)$}\end{minipage}  &  \begin{minipage}{2cm}$-5.55$ \footnotesize{$(-5.9;-5.2)$}\end{minipage}  &  \begin{minipage}{2cm}$-5.01$ \footnotesize{$(-6.7;-4.7)$}\end{minipage}  &  \begin{minipage}{2cm}$-5.49$ \footnotesize{$(-18.7;-4.9)$}\end{minipage}  &  \begin{minipage}{2cm}$-5.60$ \footnotesize{$(-14.5;-5.0)$}\end{minipage}  &  \begin{minipage}{2cm}$-6.86$ \footnotesize{$(-47.3;-5.6)$}\end{minipage}  \\ \hline
 Boosted (T=10)  &  \begin{minipage}{2cm}$-4.56$ \footnotesize{$(-4.7;-4.5)$}\end{minipage}  &  \begin{minipage}{2cm}$-5.46$ \footnotesize{$(-5.6;-5.3)$}\end{minipage}  &  \begin{minipage}{2cm}$-5.08$ \footnotesize{$(-5.8;-4.7)$}\end{minipage}  &  \begin{minipage}{2cm}$-5.04$ \footnotesize{$(-5.5;-4.6)$}\end{minipage}  &  \begin{minipage}{2cm}$-5.51$ \footnotesize{$(-5.9;-5.1)$}\end{minipage}  &  \begin{minipage}{2cm}$-5.51$ \footnotesize{$(-6.0;-5.2)$}\end{minipage}  \\ \hline
 TopKLast0.1 (T=3)  &  \begin{minipage}{2cm}$-4.98$ \footnotesize{$(-5.2;-4.7)$}\end{minipage}  &  \begin{minipage}{2cm}$-5.64$ \footnotesize{$(-6.1;-5.4)$}\end{minipage}  &  \begin{minipage}{2cm}$-5.70$ \footnotesize{$(-6.3;-5.2)$}\end{minipage}  &  \begin{minipage}{2cm}$-5.39$ \footnotesize{$(-38.4;-5.0)$}\end{minipage}  &  \begin{minipage}{2cm}$-7.00$ \footnotesize{$(-66.6;-5.4)$}\end{minipage}  &  \begin{minipage}{2cm}$-12.70$ \footnotesize{$(-44.2;-6.7)$}\end{minipage}  \\ \hline
 TopKLast0.1 (T=10)  &  \begin{minipage}{2cm}$-4.98$ \footnotesize{$(-5.3;-4.7)$}\end{minipage}  &  \begin{minipage}{2cm}$-5.57$ \footnotesize{$(-5.9;-5.3)$}\end{minipage}  &  \begin{minipage}{2cm}$-5.37$ \footnotesize{$(-6.0;-5.0)$}\end{minipage}  &  \begin{minipage}{2cm}$-5.57$ \footnotesize{$(-45.1;-4.7)$}\end{minipage}  &  \begin{minipage}{2cm}$-7.34$ \footnotesize{$(-16.1;-5.3)$}\end{minipage}  &  \begin{minipage}{2cm}$-8.86$ \footnotesize{$(-27.6;-5.5)$}\end{minipage}  \\ \hline
 TopKLast0.3 (T=3)  &  \begin{minipage}{2cm}$-4.73$ \footnotesize{$(-5.1;-4.5)$}\end{minipage}  &  \begin{minipage}{2cm}$-5.48$ \footnotesize{$(-6.0;-5.2)$}\end{minipage}  &  \begin{minipage}{2cm}$-5.22$ \footnotesize{$(-5.7;-4.8)$}\end{minipage}  &  \begin{minipage}{2cm}$-5.42$ \footnotesize{$(-21.6;-5.0)$}\end{minipage}  &  \begin{minipage}{2cm}$-5.76$ \footnotesize{$(-13.6;-5.1)$}\end{minipage}  &  \begin{minipage}{2cm}$-7.26$ \footnotesize{$(-36.2;-5.5)$}\end{minipage}  \\ \hline
 TopKLast0.3 (T=10)  &  \begin{minipage}{2cm}$-4.62$ \footnotesize{$(-4.8;-4.5)$}\end{minipage}  &  \begin{minipage}{2cm}$-5.41$ \footnotesize{$(-5.7;-5.2)$}\end{minipage}  &  \begin{minipage}{2cm}$-4.90$ \footnotesize{$(-5.2;-4.7)$}\end{minipage}  &  \begin{minipage}{2cm}$-5.24$ \footnotesize{$(-5.8;-4.9)$}\end{minipage}  &  \begin{minipage}{2cm}$-5.71$ \footnotesize{$(-6.2;-5.1)$}\end{minipage}  &  \begin{minipage}{2cm}$-5.75$ \footnotesize{$(-7.4;-5.1)$}\end{minipage}  \\ \hline
 TopKLast0.5 (T=3)  &  \begin{minipage}{2cm}$-4.59$ \footnotesize{$(-4.9;-4.4)$}\end{minipage}  &  \begin{minipage}{2cm}$-5.29$ \footnotesize{$(-5.7;-5.2)$}\end{minipage}  &  \begin{minipage}{2cm}$-5.41$ \footnotesize{$(-5.9;-4.9)$}\end{minipage}  &  \begin{minipage}{2cm}$-5.48$ \footnotesize{$(-18.5;-5.0)$}\end{minipage}  &  \begin{minipage}{2cm}$-5.82$ \footnotesize{$(-15.6;-5.2)$}\end{minipage}  &  \begin{minipage}{2cm}$-6.78$ \footnotesize{$(-18.7;-6.0)$}\end{minipage}  \\ \hline
 TopKLast0.5 (T=10)  &  \begin{minipage}{2cm}$-4.59$ \footnotesize{$(-4.8;-4.5)$}\end{minipage}  &  \begin{minipage}{2cm}$-5.35$ \footnotesize{$(-5.6;-5.2)$}\end{minipage}  &  \begin{minipage}{2cm}$-5.12$ \footnotesize{$(-5.5;-4.9)$}\end{minipage}  &  \begin{minipage}{2cm}$-5.35$ \footnotesize{$(-5.6;-4.8)$}\end{minipage}  &  \begin{minipage}{2cm}$-5.34$ \footnotesize{$(-5.8;-4.9)$}\end{minipage}  &  \begin{minipage}{2cm}$-6.00$ \footnotesize{$(-6.3;-5.6)$}\end{minipage}  \\ \hline
 TopKLast0.7 (T=3)  &  \begin{minipage}{2cm}$-4.56$ \footnotesize{$(-4.7;-4.4)$}\end{minipage}  &  \begin{minipage}{2cm}$-5.37$ \footnotesize{$(-5.5;-5.2)$}\end{minipage}  &  \begin{minipage}{2cm}$-5.05$ \footnotesize{$(-11.1;-4.7)$}\end{minipage}  &  \begin{minipage}{2cm}$-5.63$ \footnotesize{$(-43.1;-5.0)$}\end{minipage}  &  \begin{minipage}{2cm}$-5.99$ \footnotesize{$(-24.8;-5.4)$}\end{minipage}  &  \begin{minipage}{2cm}$-7.76$ \footnotesize{$(-25.2;-5.9)$}\end{minipage}  \\ \hline
 TopKLast0.7 (T=10)  &  \begin{minipage}{2cm}$-4.52$ \footnotesize{$(-4.7;-4.5)$}\end{minipage}  &  \begin{minipage}{2cm}$-5.29$ \footnotesize{$(-5.4;-5.2)$}\end{minipage}  &  \begin{minipage}{2cm}$-5.05$ \footnotesize{$(-6.6;-4.7)$}\end{minipage}  &  \begin{minipage}{2cm}$-5.38$ \footnotesize{$(-5.9;-5.1)$}\end{minipage}  &  \begin{minipage}{2cm}$-5.77$ \footnotesize{$(-6.3;-5.3)$}\end{minipage}  &  \begin{minipage}{2cm}$-6.10$ \footnotesize{$(-6.4;-6.0)$}\end{minipage}  \\ \hline
  \end{tabular}}
\end{center}

\caption{Reweighting similar to ``Cascade GAN'' from \cite{wang2016ensembles}, i.e. keep the top $r$ fraction of examples,
 based on the discriminator corresponding to the previous generator. The mixture weights are all equal (i.e. $\beta=1/t$). 
   The reported scores are the median and interval defined by the 5\% and 95\% percentile (in parenthesis) (see Section \ref{sec:metrics}), over 35 runs for each setting.
   The top table reports the coverage $C$, probability mass of $P_d$ covered
   by the $5th$ percentile of $P_g$ defined in Section \ref{experiments}. The bottom table reports the 
   log likelihood of the true data under $P_g$.
   }

\label{table:top_k_last}
\end{table}


\begin{table}

\begin{center}
\ \hspace*{-1.5cm}
{\renewcommand{\arraystretch}{1.7}
\begin{tabular}{ | l | c|c|c|c|c|c | }
  \hline
     &   $Modes: 1$  & $Modes: 2$  & $Modes: 3$  & $Modes: 5$  & $Modes: 7$  & $Modes: 10$  \\ \hline
 Vanilla  &  \begin{minipage}{2cm}$0.97$ \footnotesize{$(0.9;1.0)$}\end{minipage}  &  \begin{minipage}{2cm}$0.77$ \footnotesize{$(0.5;1.0)$}\end{minipage}  &  \begin{minipage}{2cm}$0.65$ \footnotesize{$(0.5;0.9)$}\end{minipage}  &  \begin{minipage}{2cm}$0.70$ \footnotesize{$(0.5;0.8)$}\end{minipage}  &  \begin{minipage}{2cm}$0.61$ \footnotesize{$(0.5;0.8)$}\end{minipage}  &  \begin{minipage}{2cm}$0.58$ \footnotesize{$(0.3;0.8)$}\end{minipage}  \\ \hline
 Boosted (T=3)  &  \begin{minipage}{2cm}$0.99$ \footnotesize{$(1.0;1.0)$}\end{minipage}  &  \begin{minipage}{2cm}$0.99$ \footnotesize{$(0.9;1.0)$}\end{minipage}  &  \begin{minipage}{2cm}$0.97$ \footnotesize{$(0.9;1.0)$}\end{minipage}  &  \begin{minipage}{2cm}$0.95$ \footnotesize{$(0.8;1.0)$}\end{minipage}  &  \begin{minipage}{2cm}$0.91$ \footnotesize{$(0.8;1.0)$}\end{minipage}  &  \begin{minipage}{2cm}$0.89$ \footnotesize{$(0.8;1.0)$}\end{minipage}  \\ \hline
 Boosted (T=10)  &  \begin{minipage}{2cm}$0.99$ \footnotesize{$(1.0;1.0)$}\end{minipage}  &  \begin{minipage}{2cm}$0.99$ \footnotesize{$(1.0;1.0)$}\end{minipage}  &  \begin{minipage}{2cm}$1.00$ \footnotesize{$(1.0;1.0)$}\end{minipage}  &  \begin{minipage}{2cm}$1.00$ \footnotesize{$(1.0;1.0)$}\end{minipage}  &  \begin{minipage}{2cm}$1.00$ \footnotesize{$(1.0;1.0)$}\end{minipage}  &  \begin{minipage}{2cm}$1.00$ \footnotesize{$(1.0;1.0)$}\end{minipage}  \\ \hline
 TopK0.1 (T=3)  &  \begin{minipage}{2cm}$0.98$ \footnotesize{$(0.9;1.0)$}\end{minipage}  &  \begin{minipage}{2cm}$0.98$ \footnotesize{$(0.8;1.0)$}\end{minipage}  &  \begin{minipage}{2cm}$0.91$ \footnotesize{$(0.7;1.0)$}\end{minipage}  &  \begin{minipage}{2cm}$0.84$ \footnotesize{$(0.7;1.0)$}\end{minipage}  &  \begin{minipage}{2cm}$0.80$ \footnotesize{$(0.5;0.9)$}\end{minipage}  &  \begin{minipage}{2cm}$0.60$ \footnotesize{$(0.4;0.7)$}\end{minipage}  \\ \hline
 TopK0.1 (T=10)  &  \begin{minipage}{2cm}$0.99$ \footnotesize{$(1.0;1.0)$}\end{minipage}  &  \begin{minipage}{2cm}$1.00$ \footnotesize{$(1.0;1.0)$}\end{minipage}  &  \begin{minipage}{2cm}$0.98$ \footnotesize{$(1.0;1.0)$}\end{minipage}  &  \begin{minipage}{2cm}$1.00$ \footnotesize{$(1.0;1.0)$}\end{minipage}  &  \begin{minipage}{2cm}$1.00$ \footnotesize{$(1.0;1.0)$}\end{minipage}  &  \begin{minipage}{2cm}$0.96$ \footnotesize{$(0.8;1.0)$}\end{minipage}  \\ \hline
 TopK0.3 (T=3)  &  \begin{minipage}{2cm}$0.98$ \footnotesize{$(0.9;1.0)$}\end{minipage}  &  \begin{minipage}{2cm}$0.98$ \footnotesize{$(0.9;1.0)$}\end{minipage}  &  \begin{minipage}{2cm}$0.95$ \footnotesize{$(0.9;1.0)$}\end{minipage}  &  \begin{minipage}{2cm}$0.95$ \footnotesize{$(0.8;1.0)$}\end{minipage}  &  \begin{minipage}{2cm}$0.84$ \footnotesize{$(0.6;1.0)$}\end{minipage}  &  \begin{minipage}{2cm}$0.79$ \footnotesize{$(0.5;1.0)$}\end{minipage}  \\ \hline
 TopK0.3 (T=10)  &  \begin{minipage}{2cm}$0.99$ \footnotesize{$(1.0;1.0)$}\end{minipage}  &  \begin{minipage}{2cm}$0.99$ \footnotesize{$(1.0;1.0)$}\end{minipage}  &  \begin{minipage}{2cm}$0.98$ \footnotesize{$(1.0;1.0)$}\end{minipage}  &  \begin{minipage}{2cm}$1.00$ \footnotesize{$(1.0;1.0)$}\end{minipage}  &  \begin{minipage}{2cm}$1.00$ \footnotesize{$(1.0;1.0)$}\end{minipage}  &  \begin{minipage}{2cm}$1.00$ \footnotesize{$(1.0;1.0)$}\end{minipage}  \\ \hline
 TopK0.5 (T=3)  &  \begin{minipage}{2cm}$0.99$ \footnotesize{$(0.9;1.0)$}\end{minipage}  &  \begin{minipage}{2cm}$0.99$ \footnotesize{$(1.0;1.0)$}\end{minipage}  &  \begin{minipage}{2cm}$0.96$ \footnotesize{$(0.9;1.0)$}\end{minipage}  &  \begin{minipage}{2cm}$0.98$ \footnotesize{$(0.8;1.0)$}\end{minipage}  &  \begin{minipage}{2cm}$0.88$ \footnotesize{$(0.7;1.0)$}\end{minipage}  &  \begin{minipage}{2cm}$0.88$ \footnotesize{$(0.6;1.0)$}\end{minipage}  \\ \hline
 TopK0.5 (T=10)  &  \begin{minipage}{2cm}$1.00$ \footnotesize{$(1.0;1.0)$}\end{minipage}  &  \begin{minipage}{2cm}$0.99$ \footnotesize{$(1.0;1.0)$}\end{minipage}  &  \begin{minipage}{2cm}$0.99$ \footnotesize{$(1.0;1.0)$}\end{minipage}  &  \begin{minipage}{2cm}$1.00$ \footnotesize{$(1.0;1.0)$}\end{minipage}  &  \begin{minipage}{2cm}$1.00$ \footnotesize{$(1.0;1.0)$}\end{minipage}  &  \begin{minipage}{2cm}$1.00$ \footnotesize{$(1.0;1.0)$}\end{minipage}  \\ \hline
 TopK0.7 (T=3)  &  \begin{minipage}{2cm}$0.98$ \footnotesize{$(1.0;1.0)$}\end{minipage}  &  \begin{minipage}{2cm}$0.98$ \footnotesize{$(0.9;1.0)$}\end{minipage}  &  \begin{minipage}{2cm}$0.94$ \footnotesize{$(0.8;1.0)$}\end{minipage}  &  \begin{minipage}{2cm}$0.84$ \footnotesize{$(0.8;1.0)$}\end{minipage}  &  \begin{minipage}{2cm}$0.86$ \footnotesize{$(0.7;1.0)$}\end{minipage}  &  \begin{minipage}{2cm}$0.81$ \footnotesize{$(0.7;1.0)$}\end{minipage}  \\ \hline
 TopK0.7 (T=10)  &  \begin{minipage}{2cm}$0.99$ \footnotesize{$(1.0;1.0)$}\end{minipage}  &  \begin{minipage}{2cm}$0.99$ \footnotesize{$(1.0;1.0)$}\end{minipage}  &  \begin{minipage}{2cm}$0.99$ \footnotesize{$(1.0;1.0)$}\end{minipage}  &  \begin{minipage}{2cm}$1.00$ \footnotesize{$(0.8;1.0)$}\end{minipage}  &  \begin{minipage}{2cm}$1.00$ \footnotesize{$(0.9;1.0)$}\end{minipage}  &  \begin{minipage}{2cm}$1.00$ \footnotesize{$(0.9;1.0)$}\end{minipage}  \\ \hline
  \end{tabular}}
\end{center}

\begin{center}
\ \hspace*{-1.5cm}
{\renewcommand{\arraystretch}{1.7}
\begin{tabular}{ | l | c|c|c|c|c|c | }
  \hline
     &   $Modes: 1$  & $Modes: 2$  & $Modes: 3$  & $Modes: 5$  & $Modes: 7$  & $Modes: 10$  \\ \hline
 Vanilla  &  \begin{minipage}{2cm}$-4.61$ \footnotesize{$(-5.5;-4.4)$}\end{minipage}  &  \begin{minipage}{2cm}$-5.92$ \footnotesize{$(-94.2;-5.2)$}\end{minipage}  &  \begin{minipage}{2cm}$-12.40$ \footnotesize{$(-53.1;-5.3)$}\end{minipage}  &  \begin{minipage}{2cm}$-59.62$ \footnotesize{$(-154.6;-9.8)$}\end{minipage}  &  \begin{minipage}{2cm}$-66.95$ \footnotesize{$(-191.5;-9.7)$}\end{minipage}  &  \begin{minipage}{2cm}$-63.49$ \footnotesize{$(-431.6;-14.5)$}\end{minipage}  \\ \hline
 Boosted (T=3)  &  \begin{minipage}{2cm}$-4.59$ \footnotesize{$(-4.9;-4.4)$}\end{minipage}  &  \begin{minipage}{2cm}$-5.32$ \footnotesize{$(-5.7;-5.2)$}\end{minipage}  &  \begin{minipage}{2cm}$-5.60$ \footnotesize{$(-5.8;-5.5)$}\end{minipage}  &  \begin{minipage}{2cm}$-5.40$ \footnotesize{$(-24.2;-4.5)$}\end{minipage}  &  \begin{minipage}{2cm}$-5.71$ \footnotesize{$(-14.0;-5.1)$}\end{minipage}  &  \begin{minipage}{2cm}$-6.96$ \footnotesize{$(-17.1;-5.9)$}\end{minipage}  \\ \hline
 Boosted (T=10)  &  \begin{minipage}{2cm}$-4.61$ \footnotesize{$(-4.7;-4.5)$}\end{minipage}  &  \begin{minipage}{2cm}$-5.30$ \footnotesize{$(-5.4;-5.2)$}\end{minipage}  &  \begin{minipage}{2cm}$-5.48$ \footnotesize{$(-5.6;-5.2)$}\end{minipage}  &  \begin{minipage}{2cm}$-4.84$ \footnotesize{$(-5.1;-4.3)$}\end{minipage}  &  \begin{minipage}{2cm}$-5.25$ \footnotesize{$(-5.9;-4.8)$}\end{minipage}  &  \begin{minipage}{2cm}$-5.95$ \footnotesize{$(-6.1;-5.5)$}\end{minipage}  \\ \hline
 TopK0.1 (T=3)  &  \begin{minipage}{2cm}$-4.93$ \footnotesize{$(-5.3;-4.7)$}\end{minipage}  &  \begin{minipage}{2cm}$-5.85$ \footnotesize{$(-6.2;-5.4)$}\end{minipage}  &  \begin{minipage}{2cm}$-5.38$ \footnotesize{$(-5.7;-5.0)$}\end{minipage}  &  \begin{minipage}{2cm}$-5.34$ \footnotesize{$(-5.8;-4.8)$}\end{minipage}  &  \begin{minipage}{2cm}$-5.79$ \footnotesize{$(-32.1;-5.2)$}\end{minipage}  &  \begin{minipage}{2cm}$-7.09$ \footnotesize{$(-20.7;-5.9)$}\end{minipage}  \\ \hline
 TopK0.1 (T=10)  &  \begin{minipage}{2cm}$-4.60$ \footnotesize{$(-4.8;-4.5)$}\end{minipage}  &  \begin{minipage}{2cm}$-5.47$ \footnotesize{$(-5.7;-5.3)$}\end{minipage}  &  \begin{minipage}{2cm}$-4.81$ \footnotesize{$(-5.1;-4.7)$}\end{minipage}  &  \begin{minipage}{2cm}$-4.90$ \footnotesize{$(-5.3;-4.2)$}\end{minipage}  &  \begin{minipage}{2cm}$-4.85$ \footnotesize{$(-5.6;-4.1)$}\end{minipage}  &  \begin{minipage}{2cm}$-4.57$ \footnotesize{$(-5.3;-4.2)$}\end{minipage}  \\ \hline
 TopK0.3 (T=3)  &  \begin{minipage}{2cm}$-4.65$ \footnotesize{$(-4.9;-4.4)$}\end{minipage}  &  \begin{minipage}{2cm}$-5.40$ \footnotesize{$(-5.9;-5.3)$}\end{minipage}  &  \begin{minipage}{2cm}$-4.98$ \footnotesize{$(-6.2;-4.7)$}\end{minipage}  &  \begin{minipage}{2cm}$-5.25$ \footnotesize{$(-11.4;-4.7)$}\end{minipage}  &  \begin{minipage}{2cm}$-5.96$ \footnotesize{$(-28.0;-5.5)$}\end{minipage}  &  \begin{minipage}{2cm}$-7.34$ \footnotesize{$(-25.4;-5.9)$}\end{minipage}  \\ \hline
 TopK0.3 (T=10)  &  \begin{minipage}{2cm}$-4.56$ \footnotesize{$(-4.8;-4.5)$}\end{minipage}  &  \begin{minipage}{2cm}$-5.32$ \footnotesize{$(-5.5;-5.2)$}\end{minipage}  &  \begin{minipage}{2cm}$-5.07$ \footnotesize{$(-5.9;-4.7)$}\end{minipage}  &  \begin{minipage}{2cm}$-5.08$ \footnotesize{$(-5.4;-4.5)$}\end{minipage}  &  \begin{minipage}{2cm}$-5.16$ \footnotesize{$(-5.9;-4.9)$}\end{minipage}  &  \begin{minipage}{2cm}$-5.82$ \footnotesize{$(-6.2;-5.3)$}\end{minipage}  \\ \hline
 TopK0.5 (T=3)  &  \begin{minipage}{2cm}$-4.60$ \footnotesize{$(-4.8;-4.5)$}\end{minipage}  &  \begin{minipage}{2cm}$-5.34$ \footnotesize{$(-5.6;-5.2)$}\end{minipage}  &  \begin{minipage}{2cm}$-5.34$ \footnotesize{$(-5.7;-5.0)$}\end{minipage}  &  \begin{minipage}{2cm}$-5.42$ \footnotesize{$(-19.0;-5.0)$}\end{minipage}  &  \begin{minipage}{2cm}$-5.59$ \footnotesize{$(-34.7;-4.9)$}\end{minipage}  &  \begin{minipage}{2cm}$-6.15$ \footnotesize{$(-14.8;-5.6)$}\end{minipage}  \\ \hline
 TopK0.5 (T=10)  &  \begin{minipage}{2cm}$-4.59$ \footnotesize{$(-4.7;-4.5)$}\end{minipage}  &  \begin{minipage}{2cm}$-5.31$ \footnotesize{$(-5.4;-5.2)$}\end{minipage}  &  \begin{minipage}{2cm}$-5.13$ \footnotesize{$(-5.5;-4.9)$}\end{minipage}  &  \begin{minipage}{2cm}$-5.35$ \footnotesize{$(-5.7;-4.8)$}\end{minipage}  &  \begin{minipage}{2cm}$-5.33$ \footnotesize{$(-5.8;-4.8)$}\end{minipage}  &  \begin{minipage}{2cm}$-5.72$ \footnotesize{$(-6.2;-5.3)$}\end{minipage}  \\ \hline
 TopK0.7 (T=3)  &  \begin{minipage}{2cm}$-4.60$ \footnotesize{$(-5.0;-4.4)$}\end{minipage}  &  \begin{minipage}{2cm}$-5.44$ \footnotesize{$(-5.6;-5.2)$}\end{minipage}  &  \begin{minipage}{2cm}$-5.62$ \footnotesize{$(-6.0;-5.4)$}\end{minipage}  &  \begin{minipage}{2cm}$-5.49$ \footnotesize{$(-22.2;-5.0)$}\end{minipage}  &  \begin{minipage}{2cm}$-5.64$ \footnotesize{$(-27.7;-5.3)$}\end{minipage}  &  \begin{minipage}{2cm}$-7.17$ \footnotesize{$(-22.5;-6.0)$}\end{minipage}  \\ \hline
 TopK0.7 (T=10)  &  \begin{minipage}{2cm}$-4.59$ \footnotesize{$(-4.7;-4.5)$}\end{minipage}  &  \begin{minipage}{2cm}$-5.34$ \footnotesize{$(-5.5;-5.2)$}\end{minipage}  &  \begin{minipage}{2cm}$-5.51$ \footnotesize{$(-5.6;-5.4)$}\end{minipage}  &  \begin{minipage}{2cm}$-5.35$ \footnotesize{$(-5.8;-5.0)$}\end{minipage}  &  \begin{minipage}{2cm}$-5.32$ \footnotesize{$(-6.0;-5.1)$}\end{minipage}  &  \begin{minipage}{2cm}$-6.11$ \footnotesize{$(-6.4;-5.9)$}\end{minipage}  \\ \hline
  \end{tabular}}
\end{center}

\caption{Reweighting using the top $r$ fraction of examples, based on the discriminator corresponding to
the mixture of all previous generators. The mixture weights are all equal (i.e. $\beta=1/t$).
   The reported scores are the median and interval defined by the 5\% and 95\% percentile (in parenthesis) (see Section \ref{sec:metrics}), over 35 runs for each setting.
   The top table reports the coverage $C$, probability mass of $P_d$ covered
   by the $5th$ percentile of $P_g$ defined in Section \ref{experiments}. The bottom table reports the 
   log likelihood of the true data under $P_g$.
   }

\label{table:top_k_all}
\end{table}

\end{document}